%% file: main_final.tex
\documentclass{article}

\usepackage{microtype}
\usepackage{graphicx}
\usepackage{booktabs} 

\usepackage{hyperref}

\usepackage[utf8]{inputenc} 
\usepackage[T1]{fontenc}    
\usepackage{hyperref}       
\usepackage{url}            
\usepackage{booktabs}       
\usepackage{amsfonts}       
\usepackage{nicefrac}       
\usepackage{microtype}      
\usepackage{graphicx}
\usepackage{amsmath}
\usepackage{amsthm}
\usepackage{xcolor}
\usepackage{bbm}
\usepackage[english]{babel}
\usepackage{tikz}

\makeatletter
\newtheorem*{rep@theorem}{\rep@title}
\newcommand{\newreptheorem}[2]{%
\newenvironment{rep#1}[1]{%
 \def\rep@title{#2 \ref{##1}}%
 \begin{rep@theorem}}%
 {\end{rep@theorem}}}
\makeatother

\newtheorem{example}{Example}
\newtheorem{observation}{Observation}
\newtheorem{rem}{Remark}

\input{helpers}

\usepackage[accepted]{icml2020}

\icmltitlerunning{Optimizing Black-box Metrics with Adaptive Surrogates}

\begin{document}

\twocolumn[
\icmltitle{Optimizing Black-box Metrics with Adaptive Surrogates}



\icmlsetsymbol{equal}{*}

\begin{icmlauthorlist}
\icmlauthor{Qijia Jiang}{stan}
\icmlauthor{Olaoluwa Adigun}{usc}
\icmlauthor{Harikrishna Narasimhan}{goo}
\icmlauthor{Mahdi Milani Fard}{goo}
\icmlauthor{Maya Gupta}{goo}
\end{icmlauthorlist}

\icmlaffiliation{goo}{Google Research, USA}
\icmlaffiliation{usc}{University of Southern California}
\icmlaffiliation{stan}{Stanford University}

\icmlcorrespondingauthor{Harikrishna Narasimhan}{hnarasimhan@google.com}

\icmlkeywords{Black-box metrics, Projected Gradient Descent, Finite-difference}

\vskip 0.3in
]



\printAffiliationsAndNotice{}  

\begin{abstract}
We address the problem of training models with black-box and hard-to-optimize metrics by expressing the metric as a monotonic function of a small number of easy-to-optimize surrogates. We pose the training problem as an optimization over a relaxed surrogate space, which we solve by estimating local gradients for the metric and performing inexact convex projections. We analyze  gradient estimates based on finite differences and local linear interpolations, and show convergence of our approach under smoothness  assumptions with respect to the surrogates. Experimental results on classification and ranking problems verify the proposal performs on par with methods that know the mathematical formulation, and adds notable value when the form of the metric is unknown.
\end{abstract}

\section{Introduction}
We consider the problem of training a machine learning model when the true evaluation metric is difficult to optimize on the training set. This general problem arises with many flavors and in different scenarios. For example, we may have a black-box metric whose mathematical expression is unknown or difficult to approximate with a convex training loss. The  latter is particularly true with non-decomposable evaluation metrics, such as the F-measure or ranking metrics like Precision@$K$, where it is not straight-forward to construct a differentiable objective that closely approximates the metric.

Another example is when the training labels are only a proxy for the true label. This  arises in problems where one has access to cheap-to-acquire noisy labels, such as clicks, but wishes to optimize for a more expensive label, such as whether users rate a result as good. If we have access to a small auxiliary validation set with true labels, how can this information be used to influence the training loss?  Similar examples also arise when the training data has noisy features and we have a small validation set with  clean features, or in machine learning fairness problems where the training data contains group-dependent noise, but we may have access to a small set of auxiliary clean data.



In many of the above scenarios, one wishes to optimize a black-box metric $M$  over $d$ model parameters, but does not have access to explicit gradients for $M$, nor is it practical to obtain reliable gradient estimates when $d$ is large. We provide a general solution to this problem by choosing $K \ll d$ convex surrogate losses, and expressing $M$ as an \textit{unknown} monotonic function $\psi: \R_+^K \> \R$ of the $K$  surrogates. We then reformulate the original problem as an optimization of $\psi$ over the $K$-dimensional surrogate space. The choice of surrogates can be as simple as the hinge losses on positive and negative 
samples, which should work well for metrics like the F-measure, or the surrogates can be chosen to be a family of different convex losses to handle robustness to training noise given a small set of clean validation samples. 

%
%
%

Our strategy is to estimate gradients for the unknown function $\psi$ with respect to its $K$ inputs by measuring changes in the metric $M$ and the $K$ surrogates for different perturbations on the model, and use the  estimates for $\nabla\psi$ to perform \textit{projected gradient descent} over the $K$-dimensional  surrogate space. We show how the projection step can be implemented inexactly but with convergence guarantees by solving a convex problem in the original $d$ parameters. We are thus able to \textit{adaptively} combine the $K$ surrogates to align well with the target metric $M$.

The main contributions of this paper include:
\vspace{-8pt}
\begin{enumerate}
\itemsep-0.1em 
\item A novel formulation that poses the problem of optimizing a black-box metric as a lower-dimensional problem in a surrogate space.
\item A projected gradient descent based training algorithm using finite-differences and local linear interpolations to estimate  gradients.
\item Theoretical results showing convergence to a  stationary point 
under smoothness assumptions on $\psi$.
\item Experiments showing that the proposed approach works as well as methods that take advantage of the form of the metric if known, but can give substantial gains  
when the metric truly is a black-box. 
\end{enumerate}


\section{Related Work}

There has been much work on directly  optimizing specialized classes of evaluation metrics during training. These include
approaches that relax the metric using convex surrogates  \citep{Joachims:2005,Kar+14,Narasimhan+15b,Kar+16}, 
plug-in or post-shift methods that tune a threshold on estimates of class probabilities \citep{Ye+12,Koyejo+14,Narasimhan+14,Yan+18}, reduction approaches that formulate a sequence of cost-sensitive learning tasks \citep{Parambath+14,Narasimhan+15,Alabi+18, Narasimhan18}, and  approaches that use constrained optimization and game-based formulations \citep{Eban+17,Narasimhan+19}.

However, all the above approaches require the evaluation metric to be available in closed-form. Of these, the closest to ours is the approach of \citet{Narasimhan+15}, which reformulates the  learning problem as an optimization problem over the space of confusion matrices. To ensure the constraint set is convex, this approach requires the use of stochastic classifiers, and the  theoretical guarantees assume that the metrics are convex or pseudo-convex in the  confusion matrix. In contrast, we do not require stochastic classifiers, and can handle general metrics.

Recently, there has been some work on optimizing evaluation metrics that are only available as a black-box. \citet{Zhao+19} approximate black-box metrics with a weighted training loss where the weighting function acts on a low-dimensional embedding of each example, and a validation set is used to estimate the parameters of the example-weighting function. A related approach by \citet{Ren+18} uses meta-gradient descent  to re-weight the training examples to handle training
set biases and label noise.  In contrast we model the unknown metric as a function of surrogate losses, and directly estimate the metric gradients, rather than estimating a weighting function on each example.  

\citet{Huang+19} also propose jointly adaptively learning a metric with the model training. They use a parametric form for their learned metric, whereas we nonparametrically estimate the metric gradients.  They use reinforcement learning to  align the training objective's optimum with that of the true metric, 
whereas we use gradient descent over a surrogate space. 
They do not provide any theoretical guarantees.

\citet{Grabocka+19} express the metric as a set function that 
maps each prediction to an embedding
and maps the average embedding across all examples to the predicted metric. 
They jointly optimize the parameters for the loss and the model. This approach is similar to ours in that it expresses the metric as a function on surrogate losses, and attempts to learn that function. However, our approach is different in two key points. First, we take as \emph{given} known-useful surrogate losses,
whereas they \emph{learn} decomposable
surrogate mappings from scratch. 
Second, they parameterize their surrogate functions and final mapping as neural networks, whereas we nonparametrically adaptively estimate the local gradients. They provide limited theoretical guarantees. 

Similar to \citet{Grabocka+19}, the work of \citet{Wu+18} also learns a parameterized metric (e.g. as a neural network). An auxiliary parametric ``teacher" model is used to adaptively learn the parameters for the metric that will maximize performance on a validation set. They do not provide theoretical guarantees.  

\section{Problem Setup and High-level Approach}
\label{sec:formulation}
Let $\cX$ be some instance space and  $\cY$ be  the label space. 
Let $f_\theta: \cX \> \R$ be a  model parametrized by $\theta \in \thetaspace$
that outputs a score $f_\theta(\bx)$ for instance $\bx  \in  \cX$. One can use this score to make a prediction; e.g. for binary classification problems, one predicts $\sign(f_\theta(\bx))$.
We measure performance w.r.t.\ a test distribution $D$ over $\cX \times \cY$. We consider two scenarios, one where we are provided a training sample $S$ of $n$ examples directly drawn from $D$, and the other where the training sample $S$ is drawn from a noisy distribution, and we are provided a smaller clean validation set from $D$. 

The performance of $f_\theta$ is evaluated by a metric $M: \thetaspace \> [0,1]$ computed on $D$, 
where $M$ may be as simple as the error rate $M_{err}(\theta) = \E_{(\bx, y)\sim D}\left[yf_\theta(\bx) > 0\right]$ (or an estimate), 
or $M$ may be a complex, non-decomposable metric such as Precision@$K$ that depends on the scores and the distribution in a more intricate manner.
We consider settings where the form of $M$ is unknown, and the metric is available only as a black-box, i.e.,\ for a given $\theta \in \thetaspace$, we can evaluate $M(\theta)$. The goal is to learn a good $f_\theta$ by solving:
\begin{equation}
\min_{\theta \in \thetaspace}\, M(\theta).
\label{eq:opt}
\vspace{-5pt}
\end{equation}


\subsection{Reformulation with Surrogates}
\label{sec:re-formulation}
To optimize (\ref{eq:opt}), one could directly estimate gradients of $M$ with respect to the $d$ parameters, but $d$ is usually too large for that to be practical. To relax (\ref{eq:opt}) to a more tractable problem, we take as given $K$ \textit{convex} surrogate loss functions $\ell_1, \ldots, \ell_K \colon \thetaspace \> \R_+$ where $K \ll d$, and express $M$ as an \textit{unknown} non-decreasing function of the $K$ surrogates, with an \textit{unknown} slack:  
\begin{align*}
M(\theta) = \psi(\ell_1(\theta), \ldots, \ell_K(\theta)) \,+\, \epsilon(\theta),
\end{align*}
where  $\psi\colon \mathbb{R}_+^K \rightarrow [0,1]$ is \textit{monotonic} but possibly non-convex, and the slack $\epsilon\colon \mathbb{R}^d \rightarrow [-1,1]$ determines how well the metric can be approximated by the $K$ surrogates. 
Note that this decomposition of $M$ 
is not unique. Our results hold for any such decomposition, but to enable a tighter analysis we consider a $\psi$ for which the associated worst-case slack over all $\theta$, i.e.,\ $\max_{\theta\in \R^d} |\epsilon(\theta)|$ is the minimum.

Here are examples of target metrics and convex surrogates.

\begin{example}[\textbf{Classification Metrics}]
\label{exmp:gmean}
\emph{
Consider the task of minimizing the  G-mean metric given by
$
 1 - \sqrt{\TPR \times \TNR},
$
where TPR is the true positive rate and TNR is the true negative rate. This metric promotes high accuracies on both the positive and negative class and is popular for classification tasks where there is class imbalance \citep{Daskalaki+06}. Possible surrogates for this metric include the average logistic or hinge losses on the positive and negatives examples as these serve as proxies for the TPR and TNR. It is reasonable to assume monotonic $\psi$ here, since lower surrogate values tend to produce better TPR and TNR values, and in turn lower G-means. The F-measure is another popular metric that can be written as a monotonic function of the TPR and TNR \cite{Koyejo+14}, and there again the average positive and negative losses would make good surrogates.}
\end{example}

\begin{example}[\textbf{Misaligned Training Data}] 
\emph{
Consider minimizing a metric using a training dataset that is noisy or misaligned with the test distribution, but we have access to a small validation set with clean data. The metric $M$ here is evaluated on the clean validation set, and the surrogates $\ell_1, \ldots, \ell_K$ might be convex $l_p$ losses on the training data with different values of $p>1$ to tune the noise robustness. In this case, the precise mathematical relationship $\psi$ between the validation metric and the  surrogates is unknown.
}
\end{example}

\begin{example}[\textbf{ML Fairness Problems}] 
\emph{
For blackbox ML fairness metrics, good surrogates might be logistic losses on the positive and negative samples for different groups.
}
\end{example}

\begin{example}[\textbf{Ranking Metrics}]
\emph{
Consider optimizing a ranking metric such as precision@$K$. While there are different convex surrogates available for this metric \cite{Joachims05, Narasimhan+15b}, the surrogate that performs the best can vary with the application. We have also observed in practice that sometimes setting a different value of $K$ in the training loss produces a better precision@$K$ during evaluation time.
The proposed set-up gives us a way to combine multiple available ranking surrogates (possibly with different $K$ values) to align well with the test metric. 
}
\end{example}


\subsection{High-level Approach}
Let $\cL := \{(\ell_1(\theta), \ldots, \ell_K(\theta))\,|\,\theta \in\R^d\}$ be the set of feasible surrogate profiles. We then seek to approximate \eqref{eq:opt} by ignoring the slack $\epsilon$ and posing the problem as an optimization of $\psi$ over the $K$-dimensional set $\cL$:
\begin{equation}
\min_{\bell \in \cL}\, \psi(\bell).
\label{eq:hippo}
\end{equation}
Our high-level idea is to solve this re-formulated problem by applying \textit{projected gradient descent} over $\cL$.

However, there are many challenges in implementing this idea. First, while each $\ell_k$ is convex, the space of feasible surrogates $\cL$ is not necessarily a convex set. Second, the function $\psi$ is unknown to us, and therefore we  need to estimate gradients for $\psi$ with only access to the metric $M$ and the surrogates $\bell$. Third, we would need to implement projections onto the $K$-dimensional surrogate space without explicitly constructing this set. 




\vspace{-3pt}
\section{Surrogate Projected Gradient Descent}
\label{sec:pgd}
We now explain how we tackle the above challenges. 
\vspace{-5pt}
\subsection{Convexifying the Surrogate Space}  To turn (\ref{eq:hippo}) into a problem over a convex domain, we define the \textit{epigraph} of the convex surrogate function profiles:
\[
\cU := \{\bu \in \R_+^K~|~ \bu \geq \bell(\theta) ~\text{for some}~ \theta \in \thetaspace\}\, .
\]
\begin{observation}
\label{obs:u-convex}
$\cU$ is a convex superset of $\cL$.
\end{observation}
 \vspace{-4pt}
We then optimize $\psi$ over this  $K$-dimensional convex set:
\begin{equation}
\min_{\bu \in \cU}\, \psi(\bu).
\label{eq:surrogate-opt}
\end{equation} 
This relaxation preserves the optimizer for \eqref{eq:hippo} because $\psi$ is monotonic  and $\cU$ consists of upper bounds on surrogate profiles in $\cL$: 
\begin{observation}
\label{obs:u-opt}
For any $\bu^* \in \argmin{\bu \in \cU}\,\psi(\bu)$, there exists $\bell^* \in \cL, \, \bell^* \leq \bu^*$, such that $\psi(\bell^*) = \psi(\bu^*)$.\vspace{-1pt}
\end{observation}



\subsection{Projected Gradient Descent over $\cU$}
We then perform projected gradient descent over 
$\cU$. 
We maintain iterates $\bu^t$ in $\cU$, and at each step, (i) estimate the gradient 
of $\psi$ w.r.t.\ the $K$-dimensional point $\bu^t$, (ii) perform a descent step: 
$\tilde{\bu}^{t+1} = \bu^{t} - \eta\nabla\psi(\bu^t)$, for some $\eta > 0$, and (iii) project  $\tilde{\bu}^{t+1}$ onto $\cU$ to
get the next iterate $\bu^{t+1}$.


In order to implement these steps without knowing $\psi$, or having direct access to the set $\cU$, we simultaneously maintain iterates $\theta^t$ in the original parameter space that map to iterates $\bu^t \in \cU$, i.e., for which $\bu^t =\bell(\theta^t)$.

Now to estimate gradients without direct access to $\psi$, we measure changes in the $K$ surrogates $\bell(\cdot)$ and changes in the metric $M(\cdot)$ at different perturbations of $\theta^t$ 
and compute estimates of $\nabla\psi(\bu^t))$ 
based on finite-differences or local linear interpolations. 
To compute projections without direct access to $\cU$, we formulate a convex optimization  problem over the original parameters $\theta$, and show that this results in an over-constrained projection onto $\cU$.

Thus we maintain iterates $(\bu^t, \theta^t)$ such that $\bu^t = \bell(\theta^t)$, and execute the following at every iteration:
$$
\tilde{\bu}^{t+1} = \bu^t \,-\, \eta\, \text{{gradient}}_\psi(\theta^t;\, M, \bell)
\vspace{-5pt}
$$
$$
(\bu^{t+1}, \theta^{t+1}) = \text{{project}}_{\cU}(\tilde{\bu}^{t+1};\, \bell).
$$
Figure \ref{fig:surrogate-pgd} gives a schematic description of the updates.
The gradient computation takes the current $\theta^t$ as input and probes $M$ and $\bell$ to return an estimate of $\nabla\psi(\bu^t)$. We elaborate on how we estimate gradients  in Section \ref{sec:grad-estimation}. The projection computation takes the updated $\tilde{\bu}^{t+1}$ as input and 
returns a point $\bu^{t+1}$ in $\cU$ and an associated
$\theta^{t+1}$ such that $\bu^{t+1} = \bell(\theta^{t+1})$. We explain this next.


        
        
        
        
        

\begin{figure}
\centering
\includegraphics[scale=0.25]{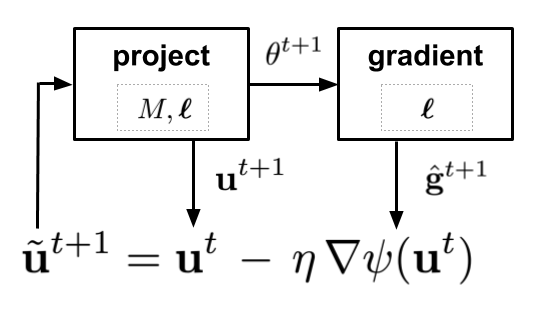}
\vspace{-6pt}
\caption{\textbf{PGD over $K$-dimensional set $\cU$}.\ `project' performs an over-constrained projection onto $\cU$.
`gradient' probes $M$ and $\ell$ returns an estimate $\hbw^{t+1} \in \R^K$ for $\nabla \psi$.}
\label{fig:surrogate-pgd}
\vspace{-8pt}
\end{figure}

\begin{figure}
    \centering
    \begin{tikzpicture}[scale=6.5]
        \fill[fill=green!20!white]
        (0,0) -- (-3mm,0mm) arc (-160:-87:3mm);
        \fill[fill=green!20!white] (-0.1mm,0) rectangle (-3mm,2mm);
        \fill[fill=green!20!white] (-0.1mm,2mm) rectangle (1mm,-2mm);
        \draw[line width=0.5mm,color=green] (-3mm,0)  arc[radius = 3mm, start angle= -160, end angle= -87];
        \draw[line width=0.5mm,color=green] (-3mm,0)  -- (-3mm, 2mm);
        \draw[line width=0.5mm,color=green] (-0.01mm,-1.97mm)  -- (1mm,-1.97mm);

        \draw[->, dashed] (-0.28,-0.05) -- (-0.38,-0.05);
        \node at (-0.4,-0.05) {$\cL$};
        \node at (-1mm,1mm) {$\cU$};
        
        \draw[->] (-4.3mm,-2.5mm) -- (1mm,-2.5mm);
        \draw[->] (-4.3mm,-2.5mm) -- (-4.3mm,2mm);
        
        \node at (1.6mm,-2.5mm) {\scriptsize $\ell_1(\theta)$};
        \node at (-4.3mm,2.3mm) {\scriptsize $\ell_2(\theta)$};
        
        \node at (-1.75mm,-1.2mm) {\scriptsize${\bu_a}$};
        \node at (-2.1mm,-1.3mm)[circle,fill,inner sep=1.5pt]{};
        \draw[->] (-2.6mm,-1.95mm) -- (-2.17mm,-1.37mm);
        
        \node at (-0.5mm,-1.95mm)[circle,fill,inner sep=1.5pt]{};
        \node at (-0.5mm,0.1mm) {\scriptsize$\tilde{\bu}_b$};
        \draw[->] (-0.5mm,-0.2mm) -- (-0.5mm,-1.85mm);
        \node at (-0.5mm,-0.2mm)[circle,fill,inner sep=1.5pt]{};
        \node at (-0.5mm,-2.3mm) {\scriptsize$\bu_b$};

        \node at (-2.62mm,-1.95mm)[circle,fill,inner sep=1.5pt]{};
        \node at (-3mm,-1.9mm) {\scriptsize$\tilde{\bu}_a$};
    \end{tikzpicture}  
    \vspace{-5pt}
    \caption{\textbf{Over-constrained projection.} The space of surrogate profiles $\cL = \{(\ell_1(\theta), \ell_2(\theta)) \,|\, \theta \in \R^d\}$ is a non-convex set (solid line), and its epigraph $\cU = \{\bu \geq \bell\,|\, \bell \in \cL\}$ is convex (shaded region). For the point $\tilde{\bu}_a$ outside $\cU$, the solution $\bu_a$ to \eqref{eq:project} is the same as the exact projection $\Pi(\tilde{\bu}_a)$ onto $\cU$. For the point $\tilde{\bu}_b$ inside the set,  $\Pi(\tilde{\bu}_b)$ = $\tilde{\bu}_b$, whereas $\bu_b$ is one of many solutions to \eqref{eq:project} on the boundary and with $\bu_b \leq  \Pi(\tilde{\bu}_b)$ in each coordinate. 
    }
    \vspace{-7pt}
    \label{fig:project}
\end{figure}

\subsection{Over-constrained Projection}
To implement the projection without explicit access to $\cU$, we set up an optimization over $\theta$ by penalizing a clipped $L_2$-distance between the surrogate profile $\bell(\theta)$  and  $\tilde{\bu}^{t+1}$: 
\begin{align}
\theta^{t+1} &\in \argmin{\theta \in \thetaspace}\,\|\big(\bell(\theta) \,-\, \tilde{\bu}^{t+1}\big)_+\|^2 \nonumber \\
\bu^{t+1} &= \bell(\theta^{t+1}), \label{eq:project}
\end{align}
where $(z)_+ := \max\{0, z\}$ is applied element-wise and $\|\cdot\|$ is the $L_2$-norm. 
Note that we penalize errors in only one direction (i.e. the errors where $\ell_k(\theta) \geq \tilde{u}^{t+1}_k$). This has the advantage of the optimization problem being convex. Moreover, as we show below, 
\eqref{eq:project} 
results in an over-constrained projection: any solution $\bu^{t+1}$ to \eqref{eq:project} is feasible (i.e.\ is in $\cU$), and for a monotonic $\psi$, yields a $\psi$-value that is no worse than what we would get with an exact projection. 
\begin{lemma}
Let $\bu^+$ be the exact projection of $\tilde{\bu}^{t+1} \in \R^K_+$ onto $\cU$. For any solution $\bu^{t+1}$ to \eqref{eq:project}, we have
$\bu^{t+1} \in \cU$, $\bu^{t+1} \leq \bu^+$, and for a monotonic $\psi$,  $\psi(\bu^{t+1}) \leq \psi(\bu^+)$. 
 \label{lem:u-project}
\end{lemma}

 \begin{figure}
 \vspace{-12pt}
 \begin{algorithm}[H]
\caption{Surrogate Projected Gradient Descent}
\label{algo:pgd}
\begin{algorithmic}[1]
   \STATE \textbf{Input:} Black-box metric $M$, surrogate loss functions $\ell_1, \ldots, \ell_K:\mathbb{R}^d \rightarrow \mathbb{R}_+^K$, hyper-parameters: $T, \eta$
   \STATE Initialize $\theta^1 \in \thetaspace, \bu^1 = \bell(\theta^1)$
   \FOR{$t=1$ {\bfseries to} T}
   \STATE \textbf{Gradient estimate:}\ Obtain an estimate $\hbw^t$ for gradient $\nabla\psi(\bu^t)$ by invoking Algorithms 2 or 3 with inputs $\theta^t$, $M$ and $\ell_1, \ldots, \ell_K$
    \STATE \textbf{Gradient update:} $\tilde{\bu}^{t+1} = \bu^t \,-\, \eta\, \hbw^t$
   \STATE \textbf{Over-constrained projection:} Solve:
         \vspace{-5pt}
   $$
   \theta^{t+1} \in \argmin{\theta \in \thetaspace}\,\|\big(\bell(\theta) \,-\, \tilde{\bu}^{t+1}\big)_+\|^2
     \vspace{-5pt}
   $$
   to accuracy 
   $\mathcal{O}\big(\frac{1}{\beta^2 T}\big)$
   and set $\bu^{t+1} = \bell(\theta^{t+1})$
   \ENDFOR
  \end{algorithmic}
 \end{algorithm}
 \vspace{-15pt}
\end{figure}
Problem \eqref{eq:project} may not have a unique solution. For example, when $\tilde{\bu}^{t+1}$ is in the interior of $\cU$, the exact projection $\bu^+$ is the same as $\tilde{\bu}^{t+1}$, whereas
the solutions to \eqref{eq:project} are the points $\bu$ on the boundary of $\cU$ with $\bu \leq \bu^+$
(see Figure \ref{fig:project}).
As $\psi$ is monotonic,
picking any of these solutions for the next iterate 
doesn't hurt the  convergence of the algorithm.
%
%
%
%
%
%

An outline of the projected gradient descent with this inexact projection is presented in Algorithm \ref{algo:pgd}. 
One can interpret the algorithm as \textit{adaptively} combining the $K$ surrogates $\ell_k$'s to optimize the  metric $M$ (see Appendix \ref{app:prox} for the details).
\vspace{-2pt}


\subsection{Convergence Guarantee}
We show convergence of Algorithm \ref{algo:pgd} to a stationary point of $\psi(\bell(\cdot))$. 
 Since we probe $M$ to estimate gradients for $\psi$, the errors in the estimate would depend on how closely $\psi(\bell(\cdot))$ approximates $M$, and in turn on the magnitude of the slack term $\epsilon$. We assume here that the gradient estimation error $\E\left[\|\hbw^t \,-\, \nabla\psi(\bell(\theta^t))\|^2\right]$ at each step $t$ is bounded by a $\kappa_\epsilon \in \R_+$ that depends on the slack $\epsilon$. 
 In Section \ref{sec:grad-estimation}, we present gradient estimates that satisfy this condition. 
\begin{theorem}[Convergence of Algorithm \ref{algo:pgd}] 
\label{thm:meta-result}
Let $M(\theta) = \psi(\bell(\theta)) + \epsilon(\theta)$, for a $\psi$ that is monotonic, $\beta$-smooth and $L$-Lipschitz, and the worst-case slack $\max_{\theta \in \R^d}|\epsilon(\theta)|$ is the minimum among all such decompositions of $M$.

Suppose each $\ell_k$ is $\gamma$-smooth and $\Phi$-Lipschitz in $\theta$ with $\|\bell(\theta)\|\leq G, \, \forall \theta$. Suppose the gradient estimates $\hbw^t$
  satisfy $\E\left[\|\hbw^t \,-\, \nabla\psi(\bell(\theta^t))\|^2\right] \leq \kappa_{\epsilon}, ~\forall t \in [T]$ and the projection step satisfies
$\|(\bell(\theta^{t+1}) - \tilde{\bu}^t)_+\|^2 \leq
\min_{\theta \in \thetaspace}\|(\bell(\theta)- \tilde{\bu}^t)_+\|^2 \,+\, \mathcal{O}(\frac{1}{\beta^2 T}), ~\forall t \in [T]$. Set stepsize $\eta = \frac{1}{\beta^2}$.

Then  Algorithm \ref{algo:pgd} converges to an approximate stationary point of $\psi(\bell(\cdot))$:
 \vspace{-3pt}
 \begin{align*}
 \min_{1\leq t\leq T}&\E\left[\|\nabla \psi(\bell(\theta^t))\|^2\right] \leq
C\bigg(
 \frac{\beta}{\sqrt{T}} + \sqrt{\kappa_\epsilon} + \sqrt{L}\kappa_\epsilon^{1/4} 
 \bigg),
  \vspace{-3pt}
\end{align*}
where the expectation is over the randomness in the gradient estimates, and 
$C = \cO\big(KL\big(\gamma\big(G+\frac{L}{\beta^2}\big)+\Phi^2\big)\big)$.


\end{theorem}
\begin{rem}[\textbf{Stationary point of $M$}]\,
\emph{
When the gradient estimation error  $\kappa_\epsilon$ 
is small and the number of steps $T \> \infty$, the algorithm reaches a model $\theta$ with a small gradient norm $\|\nabla\psi(\bell(\cdot))\|$. If additionally the slack term $\epsilon$ is Lipschitz in $\theta$, 
then this implies that the algorithm also converges to an approximate stationary point of  the metric $M$.
}
\vspace{-5pt}
\end{rem} 

The proof of Theorem \ref{thm:meta-result} proceeds in two parts. We first show that the algorithm converges to an approximate stationary point of $\psi$ over $\cU$. For this, we extend recent results \cite{Ghadimi2016} on convergence of projected gradient descent for smooth non-convex objectives. 
We then exploit the smoothness of the surrogates $\bell$ to show that this result translates to the algorithm converging to an approximate stationary point of $\psi(\bell(\cdot))$ w.r.t.\ $\theta$. 

\begin{rem}[\textbf{Prior convergence results}]
\emph{
A key difference between our analysis and prior works on zeroth-order gradient methods \citep{Duchi, Ghadimi2016, Nesterov+17} is that we do not directly optimize the given objective over the space of parameters $\theta$, and instead perform an optimization over a relaxed surrogate space $\cU$ that is not directly specified, and do so 
using inexact projections and approximate gradient estimates.
}
\end{rem} 

\section{Gradient Estimation Techniques}
\label{sec:grad-estimation}
We now address the issue of estimating the gradient $\hbw^t$ of $\psi$ at a given $\bell(\theta^t)$ without explicit access to $\psi$. We provide an algorithm based on finite-differences, and another  based on local linear interpolations. We also show error bounds for these algorithms, i.e.\ bound the errors $\kappa_\epsilon$ in Theorem \ref{thm:meta-result}. 

\subsection{Finite Differences}
\label{sec:fd}
We first consider the case where both the surrogates $\bell$ and metric $M$ are evaluated on the same  sample.  Let  $\boldf_\theta := [f_\theta(\bx_1), \ldots, f_\theta(\bx_n)]^\top \in \R^n $ denote the scores of the model $\theta$ computed on the $n$ training examples.
We overload notation and use $M(\boldf_\theta, \by)$ 
to denote the value of the evaluation metric $M$ on the model scores $\boldf_\theta \in \R^n$ and labels $\by \in \cY^n$. 
 Similarly, we use   $\ell_k(\boldf_\theta, \by)$ to denote the value of surrogate loss $\ell_k$ on $\boldf_\theta$ and $\by$.

We present  our method in Algorithm \ref{algo:finite-diff}. We adopt a standard finite-difference gradient estimate \citep{Nesterov+17}, which requires us to perturb the surrogates $\bell$ with random Gaussian vectors $Z^1, \ldots, Z^m \sim  \mathcal{N}(\0, \mathbf{I}_K)$, evaluate $\psi$ at the perturbed surrogate profiles, and calculate 
\[\frac{1}{m}\sum_{j=1}^m\frac{\psi(\bell(\boldf_\theta, \by) \,+\, \sigma Z^j) \,-\, \psi(\bell(\boldf_\theta, \by))}{\sigma},\] 
for $\sigma > 0$. In our case, we cannot directly perturb the surrogates $\bell$ and evaluate changes in $\psi$. 
Instead, we perturb the scores $\boldf_\theta$ so that the corresponding changes in $\bell$ follows a Gaussian distribution, and evaluate the difference between the metric $M$  at the original and perturbed scores. This is possible, for example, when each $\ell_k(\boldf_\theta, \by)$ is an average of point-wise losses $\phi_k(y_if_\theta(x_i))$ on different subsets of the data, for some invertible function $\phi_k: \R \> \R$, in which case, it is easy to compute the right  amount of perturbation to the scores $\boldf_\theta$ to produce the desired perturbation in $\ell_k$. 


\begin{lemma}[Finite difference estimate]
\label{lem:fd}
Let $M$ be as defined in Theorem \ref{thm:meta-result} and $|\epsilon(\theta)| \leq \bar{\epsilon}, \forall \theta$.
 Let $\hbw$ be returned by Algorithm \ref{algo:finite-diff} for a given $\theta'$, $m$ perturbations and $\sigma = \frac{\sqrt{\bar{\epsilon}}}{\sqrt{K}\beta^2}$. 
\begin{align*}
\E\left[\|\hbw \,-\, \nabla\psi(\bell(\theta'))\|^2\right] 
\,\leq\,
\cO\left(\frac{L^2K}{m} + \bar{\epsilon}K^2\beta^2\right),
\end{align*}
where the expectation is over the random perturbations.
\vspace{-5pt}
\end{lemma}
This gives a bound on $\kappa_\epsilon$ in Theorem \ref{thm:meta-result} when Algorithm \ref{algo:finite-diff} is used for gradient estimates. Note the error depends on the slack magnitude $\bar{\epsilon}$, and decreases with more perturbations. 

\begin{figure}
\vspace{-10pt}
 \begin{algorithm}[H]
\caption{Finite-difference Gradient Estimate}
\label{algo:finite-diff}
    \begin{algorithmic}[1]
   \STATE \textbf{Input:} $\theta' \in \mathbb{R}^d$, $M$, $\ell_1, \ldots, \ell_K$ 
   \STATE \textbf{Hyper-parameters}:
  Num\ of perturbations
  $m$, $\sigma$
   \STATE Draw $Z^1, \ldots, Z^m \sim \mathcal{N}(\0, \mathbf{I}_K)$
  \STATE Find $\Delta^j \in \R^n$ s.t.\ $\bell(\boldf_{\theta'} \,+\, \Delta^j, \by) \,=\, \bell(\boldf_{\theta'}, \by) \,+\, \sigma Z^j$, for $j=1, \ldots, m$
   \STATE $\displaystyle \hbw = \frac{1}{m}\sum_{j=1}^m\frac{M(\boldf_{\theta'} \,+\, \Delta^j, \by) \,-\, M(\boldf_{\theta'}, \by) }{\sigma}Z^j$
   \STATE \textbf{Output}:  $\hbw$ 
 \end{algorithmic}
\end{algorithm} 
\vspace{-15pt}
\begin{algorithm}[H]
\caption{Linear Interpolation Gradient Estimate}
\label{algo:ls}
    \begin{algorithmic}[1]
   \STATE \textbf{Input:} $\theta' \in \R^d$, $M$, $\ell_1, \ldots, \ell_K$ 
   \STATE \textbf{Hyper-parameters}:
  Num\ of perturbations
  $m$, $\sigma$
  \STATE Draw $Z_1^1, \ldots, Z_1^m, Z_2^1, \ldots, Z_2^m \sim \mathcal{N}(\0, \mathbf{I}_d)$
  \STATE $\bH_{j,:} = \bell(\theta' + \sigma Z_1^j) \,-\, \bell(\theta' + \sigma Z_2^j),~j = 1, \ldots, m
  $
  \STATE $\bM_{j,:} = M(\theta' + \sigma Z_1^j) \,-\, M(\theta' + \sigma Z_2^j),~j = 1, \ldots, m$
   \STATE $\displaystyle \hbw \,\in\,\argmin{\bw \in \R^K}\,
   \|\bH\bw - \bM\|^2
   $ 
   \STATE \textbf{Output}:  $\hbw$ 
 \end{algorithmic}
\end{algorithm}
\vspace{-20pt}
\end{figure}

\subsection{Local Linear Interpolations}
The finite-difference approach is not applicable to settings where the  metric is evaluated on a validation sample but the surrogates are evaluated on training examples (as in Example 2), or where finding the right amount of perturbation on the scores is difficult. For such cases we present a
local linear interpolation based approach in Algorithm \ref{algo:ls}, where we perturb the model parameters $\theta$ instead of the scores.

We use the fact that a smooth function $\psi$ can be locally approximated by a linear function, and estimate the gradient of $\psi$ of at $\bell(\theta)$  by perturbing $\theta$, measuring the corresponding differences in the surrogates $\bell$ and the metric $M$, and fitting a linear function from the surrogate differences to the metric differences. Specifically, for $d$ model parameters, 
we draw two independent sets of $d$-dimensional Gaussian perturbations $Z_1^1, \ldots, Z_1^m, Z_2^1, \ldots, Z_2^m \sim \mathcal{N}(\0, \mathbf{I}_d)$, and return a linear fit from $\bH = [\bell(\theta + \sigma Z_1^j) \,-\, \bell(\theta + \sigma Z_2^j)]_{j=1}^m$ to  $\bM = [M(\theta + \sigma Z_1^j) \,-\, M(\theta + \sigma Z_2^j)]_{j=1}^m$. 


\begin{lemma}[Linear interpolation estimate]
\label{thm:ls}
Let $M$ be defined as in Theorem \ref{thm:meta-result} and
$|\epsilon(\theta)| \leq \bar{\epsilon}, \forall \theta$. Assume each $\ell_k$ is 
$\Phi$-Lipschitz in $\theta$ w.r.t.\ the $L_\infty$-norm,
and $\|\bell(\theta)\|\leq G\,\, \forall \theta$.
Suppose for a given $\theta'$, $\sigma$ and perturbation count $m$, the expected covariance matrix for the left-hand-side of the linear system $\bH$ is well-conditioned, and has the smallest singlular value $\lambda_{\min}(\sum_{i=1}^m \E[\mathbf{H}_i\mathbf{H}_i^{\top}]) = 
\mathcal{O}(m\sigma^2\Phi^2)$.
Then  
setting $\sigma = \tilde{\cO}\left(\frac{G^{1/3}\bar{\epsilon}^{1/3}}{\Phi K^{3/2}\beta^{1/3}}\right)$ and $m = \tilde{\cO}\left(\frac{G^4K^9\beta^2}{\bar{\epsilon}^2}\right)$, 
Algorithm \ref{algo:ls} returns  w.h.p.\ (over draws of random perturbations)  a gradient estimate $\hbw$ that satisfies:
$$\|\hbw \,-\, \nabla\psi(\bell(\theta'))\|^2
\,\leq\,
\tilde{\cO}\left(G^{1/3}\bar{\epsilon}^{1/3}K^3\beta^{2/3}\right)\, .
\vspace{-5pt}
$$
\end{lemma}
We show in Appendix \ref{app:ls-proof} how this high probability statement can then be used to derive a bound on the expected errors $\kappa_\epsilon$ in Theorem \ref{thm:meta-result}. 
Prior works provide error bounds on a similar gradient estimate under an assumption that the perturbation matrix $\bH$ can be chosen to be invertible \citep{Conn08, Conn09, Berahas19}. 
In our case, however, $\bH$ is not chosen explicitly, but instead contains measurements of changes in surrogates for random perturbations on $\theta$. Hence to show an error bound, we need a slightly subtle condition on the correlation structure of the surrogates (that essentially says the variance of the perturbed surrogates are large enough and the rates are not strongly correlated with each other), which we express as a condition on the smallest singular value of the covariance of $\bH$. 

\subsection{Handling Non-smooth Metrics}
For $\psi$ that is non-smooth and Lipschitz, we extend the finite difference gradient estimate in Section \ref{sec:fd} with a two-step perturbation. We draw two sets of Gaussian vectors $Z_1^1, \ldots, Z_1^m, Z_2^1, \ldots, Z_2^m \sim \mathcal{N}(\0, \mathbf{I}_K)$ and approximately calculate
\[\frac{1}{m}\sum_{j=1}^m\frac{\psi(\bell(\boldf_\theta, \by) + \sigma_1 Z_1^j + \sigma_2 Z_2^j) - \psi(\bell(\boldf_\theta, \by) + \sigma_1Z_1^j)}{\sigma_2}\] 
for $\sigma_1, \sigma_2 > 0$, by perturbing $\bell$ through the scores $\mathbf{f}_\theta$ and measuring changes in  $M$ instead of $\psi$. 
This approach computes a finite-difference gradient estimate for a smooth approximation  to the original $\psi$, given by $\psi_{\sigma_1}(\bu) := \E\left[\psi(\bu \,+\, \sigma_1 Z_1)\right]$, where $Z_1 \sim \mathcal{N}(\0, \mathbf{I}_K)$. 
We provide error bounds  in Appendix \ref{app:non-smooth} by building on recent work by  \citet{Duchi}, and discuss asymptotic convergence of Algorithm \ref{algo:pgd} as $\sigma_1 \> 0$. 

\section{Experiments}
\label{sec:expts}
We present experiments to show the proposed approach, Algorithm \ref{algo:pgd}, 
is able to perform as well as  methods that take advantage of a metric's form where available, and is also able to provide gains for metrics that are truly a black-box. We consider a simulated classification task, fair classification with noisy features, a ranking task and classification with proxy labels. The datasets we use are listed in Table \ref{tab:datasets}. 

We use the linear interpolation approach in Algorithm \ref{algo:ls} for estimating gradients, as this is the most practical among the proposed estimation methods, and applicable when the surrogates and metrics are evaluated on different samples. We use linear models, and tune hyper-parameters such as step sizes and the perturbation parameter $\sigma$ for gradient estimation using a held-out validation set. We run the projected gradient descent with 250 outer iterations and 1000 perturbations. For the projection step, we run 100 iterations of Adagrad.
See Appendix \ref{app:expts} for more details and a discussion on perturbations.
The code has been made available.

\begin{table}
    \centering
        \vspace{-5pt}
    \caption{Datasets used in our experiments.}
    \vspace{3pt}
    \label{tab:datasets}
    \begin{tabular}{lrrr}
        \hline
        Dataset & \#instances & \#features & Groups \\
        \hline
        Simulated & 5000 & 2 & -\\
        COMPAS & 4073 & 31 & M/F
        \\
        Adult & 32561 & 122 & M/F
        \\
        Credit  & 30000 & 89 & M/F
        \\
        Business  & 11560 & 36 & C/NC
        \\
        KDD Cup 08 & 102294 &117  &-
        \\
        \hline
    \end{tabular}
    \vspace{-5pt}
\end{table}

\begin{table}[t]
\centering
\vspace{-5pt}
\caption{Test G-mean on sim.\ data. \emph{Lower} is better.}
\vspace{3pt}
\label{tab:gmean}
\begin{tabular}{lccc}
	\hline
& LogReg & PostShift & Proposed \\\hline
Simulated	 &1.000	 &0.848 &\textbf{0.803} \\
	\hline
\end{tabular}
\vspace{-10pt}
\end{table}

\begin{figure}[t]
\centering
\includegraphics[scale=0.4]{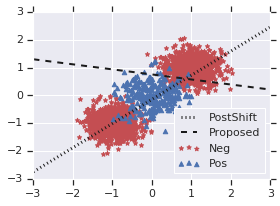}
\caption{Hyperplanes learned by proposed method and PostShift on simulated data. }
\vspace{-10pt}
\label{fig:hyperplanes}
\end{figure}

\begin{table*}[t]
\centering
\vspace{-5pt}
\caption{Average test macro F-measure across groups with clean features. \textit{Higher} is better. Despite having only black-box access to the metric, our approach performs comparable to methods that take advantage of the form of the metric.}
\vspace{3pt}
\label{tab:fm-clean}
\begin{tabular}{lccccc}
	\hline
& LogReg & PostShift & RelaxedFM & GenRates & Proposed \\\hline
Business	 &0.793 	 &0.789 	 &0.794 	 &0.793 	 &\textbf{0.796} \\
COMPAS	 &0.560 	 &\textbf{0.631} 	 &0.614 	 &0.620 	 &0.629 \\
Adult	 &\textbf{0.668} 	 &0.664 	 &{0.665} 	 &0.654 	 &{0.665} \\
Default	 &0.467 	 &\textbf{0.536} 	 &0.525 	 &0.532 	 &0.533 \\
	\hline
\end{tabular}
\vspace{-10pt}
\end{table*}

\begin{figure*}[t]
\centering
    \begin{subfigure}[b]{0.245\textwidth}
    \centering
        \includegraphics[scale=0.42]{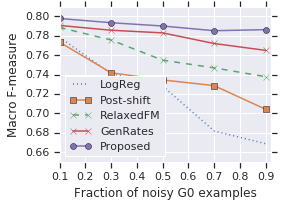}
        \caption{Business}
    \end{subfigure}\hspace{-2pt}
    \begin{subfigure}[b]{0.245\textwidth}
        \includegraphics[scale=0.43]{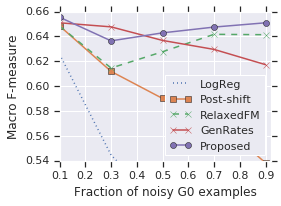}
        \caption{Adult}
    \end{subfigure}
    \begin{subfigure}[b]{0.245\textwidth}
        \includegraphics[scale=0.43]{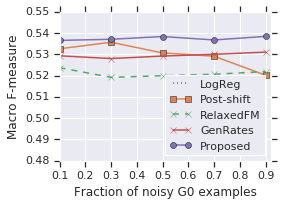}
        \caption{Default}
    \end{subfigure}
    \begin{subfigure}[b]{0.245\textwidth}
        \includegraphics[scale=0.42]{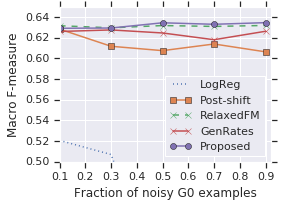}
        \caption{COMPAS}
    \end{subfigure}
    \vspace{-6pt}
\caption{Test macro F-measure across groups for varying noise levels, averaged over 5  trials. \textit{Higher} is better.}
\label{fig:fairness-noise}
\end{figure*}

\subsection{Optimizing G-mean on Simulated Data}
We first  apply our approach to maximize a non-black box evaluation metric:
$
\text{G-mean} = 1 - \sqrt{\TPR \times \TNR},
$
described in Example \ref{exmp:gmean}. 
We consider a simulated binary classification task in two dimensions, containing 10\% positives and 90\% negatives. The positive examples are drawn from a Gaussian with mean $[0,0]$ and covariance matrix $0.2 \times \boldI_2$. The negative examples are drawn from a mixture of two Gaussians centered at $[-1,-1]$ and $[1,1]$, with equal priors, and with a covariance matrix of $0.1 \times \boldI_2$.

We apply our method with two surrogate functions: the average hinge losses on the positive and negative examples.
The results are shown  in Table \ref{tab:gmean}. 
We compare against two baselines: logistic regression that optimizes a standard cross-entropy loss,  and a plug-in or post-shift approach that shifts the a threshold on the logistic regression model to optimize G-mean \citep{Narasimhan+14}. Because of the class imbalance,  logistic regression  learns to always predict the majority negative class and yields zero true positive rate and as a result a poor G-mean. Post-shift produces a better G-mean, but the proposed method performs the best. It is clear from the resulting decision boundaries  shown in Figure \ref{fig:hyperplanes} that the proposed method learns the better linear separator.


\subsection{Macro F-measure with Noisy Features}
\label{sec:expt-fm}
For this experiment we consider training a classifier with fairness goals defined on binary protected attributes. We seek to maximize the average F-measure across the groups:
\[
\text{Macro $F_1$} ~=~ \frac{1}{2}\sum_{G \in \{0,1\}} \frac{2 \times \text{Precision}_G \times \text{Recall}_G}{\text{Precision}_G + \text{Recall}_G},
\]
where $\text{Precision}_G$ and $\text{Recall}_G$ are the precision and recall on protected group $G$.  Optimizing a sum of F-measures is harder than optimizing the binary F-measure because the summation destroys its pseudo-convexity property  \citep{Narasimhan+19}. 


We use four fairness datasets: (1) \textit{COMPAS}, where the goal is to predict recidivism with \textit{gender} as the protected attribute \citep{Angwin+16};  (2) \textit{Adult}, where the goal is to predict if a person's income is more than 50K/year, and we take \textit{gender} as the protected group \citep{uci}; (3) \textit{Credit Default}, where the task is to predict whether a customer would default on his/her credit card payment, and we take \textit{gender} as the protected group \citep{uci}; 
(4) \textit{Business Entity Resolution}, a proprietary dataset from a large internet services company, where the goal is to predict whether a pair of business descriptions refer to identical businesses, and we consider \textit{non-chain} businesses  as protected. In each case, we split the data into train-validation-test sets in the ratio $4/9:2/9:1/3$.

\textbf{Training with no noise.}\ The first set of experiments tests if the proposed approach is able to match the  performance of existing methods that are customized to optimize the macro F-measure. We compare against (i) plain logistic regression method, (ii) a plug-in or post-shift method that tunes a threshold  on the logistic regression model to maximize the F-measure \citep{Koyejo+14, Narasimhan+14}, (iii) an approach that optimizes a continuous relaxation to the F-measure that replaces the indicators with the hinge loss, and (iv) the recent ``generalized rates'' approach of \citet{Narasimhan+19} for optimizing metrics that are a sum of ratios.  
We apply our approach using four surrogate losses, each one is the hinge loss averaged over either the positive or negative examples, calculated separately for each of the two groups. As seen in Table \ref{tab:fm-clean}, despite having only black-box access to the metric, the proposed approach performs comparable to 
the other methods that are directly tailored to optimize the macro F-measure.

\textbf{Training with noisy features.}
 The second set of experiments evaluates the performance of these methods when the training set has noisy features for just one of the groups, while the smaller validation set contains clean features. We use our approach to adaptively combine the same four surrogate losses computed on the noisy training set to best optimize the macro F-measure on the clean validation set. 

We chose a certain fraction of the examples at random from one of the groups, which we refer to as group 0, and for these examples, we add Gaussian noise to the real features (with mean 0 and the same standard deviation as the feature), and flip the binary features with probability 0.9. Figure \ref{fig:fairness-noise} shows the test F-measure for the different methods with varying fraction of noisy examples in group 0. Except for logistic regression, all other methods have access to the validation set: post-shift uses the validation set to tune a threshold on the logistic regression model; the RelaxedFM and GenRates method optimize their loss on the training set, but pick the best model iterate using the validation set. The proposed approach is able to make the best use of the validation set, and consistently performs the best across most noise levels.


\subsection{Ranking to Optimize PRBEP}
\label{sec:expt-ranking}
We next consider a ranking task, where the goal is to learn a scoring function $f$ that maximizes
the precision-recall break-even point (PRBEP), i.e. yields maximum precision at the threshold where precision and recall are equal. PRBEP is a special case of Precision@$K$ when $K$ is set to the number of positive examples in the dataset.
For this task, we experiment with the KDD Cup 2008 breast cancer detection data set \citep{Rao+08} popularly used in this literature \citep{Kar+15, Mackey+18}.  We randomly split this dataset 60/20/20 for training, validation, and test.

Since the break-even point for a dataset is not known before-hand, we use surrogates that approximate precision at different recall thresholds $\tau$.
We use the quantile-based surrogate losses of \citet{Mackey+18} 
with $\tau = 0.25, 0.5, 0.75$. 
As a comparison, we optimize the avg-precision@$K$ surrogate provided by \citet{Kar+15}. As seen in Table \ref{tab:prbep},  the proposed approach is able to learn a better training loss by combining the three quantile surrogates, and yields the best PRBEP on the both the training and test sets. 

\begin{table}[t]
\centering
\caption{Train and test PRBEP on KDD Cup 2008 data. \textit{Higher} is better. 
}
\vspace{3pt}
\label{tab:prbep}
\begin{tabular}{lccc}
	\hline
& LogReg &  \citet{Kar+15}  & Proposed \\\hline
Train & 0.480 & 0.473 &\textbf{0.546} \\
Test & 0.472 & 0.441 &\textbf{0.480} \\
	\hline
\end{tabular}
\end{table}

\begin{table}[t]
\centering
\caption{Test classification error where the training labels are only proxy labels with unknown relationship to the true labels. The proposed method was run with both hinge and sigmoid surrogates. \textit{Lower} is better. 
} 
\vspace{3pt}
\label{tab:proxy}
\begin{tabular}{lccccc}
	\hline
 & LogReg & PostShift & Hinge & Sigmoid
 \\\hline
 Adult
& 0.333
& 0.322
& \textbf{0.314}
& \textbf{0.314}
\\
Business 
& 0.340
& {0.251}
& 0.256
& \textbf{0.236}
\\
	\hline
\end{tabular}
\vspace{-10pt}
\end{table}

\subsection{Classification with Proxy Labels}
\label{sec:expt-proxy}
Next, we consider classification tasks where the training labels are proxies for the true labels, but the validation data has the true labels. We seek to minimize the classification error on the validation set by combining hinge loss surrogates evaluated separately on the positive and negative training examples. While the theory requires convex losses, we experiment with also running the algorithm with non-convex sigmoid losses as surrogates. 


For the Adult data, we predict whether a candidate's gender is female, and take the marital-status-wife feature as the proxy label. For the Business Entity Resolution data, we predict whether a pair of business descriptions refer to the same business, and use the has-same-phone-number feature as a proxy label. 

We compare with a logistic regression model trained with the proxy labels and a post-shift method that corrects the logistic regression threshold to minimize classification error on the validation data. As expected logistic regression  yields the highest test error. On Adult, both variants of the proposed method are better than PostShift. On Business, 
the proposed method performs slightly worse than PostShift when run with hinge surrogates, but yields notable improvements when run with sigmoid surrogates, which are tighter relaxations to the true errors.


\section{Discussion}
There is currently a lot of interest 
in training 
models with better alignment with evaluation metrics. Here, we have investigated a simple method that directly estimates only the needed gradients for gradient descent training, and does not require assuming a parametric form. This simplicity enabled us to provide rigorous theoretical guarantees. 

Experimentally, 
our approach was as good as strategies that take advantage of a metric's form (where available), 
gave notable gains over baselines for black-box ranking, and was  significantly better than post-shifting for experiments with group-dependent noise.
%
%
For the proxy label experiments, however, the results 
were mixed, with the proposed method requiring a tighter surrogate relaxation to perform better than post-shift. 
Post-shift is a strong baseline --  in theory, for many metrics it is optimal to simply post-shift the Bayes class probability model $\P(y=1|x)$
with a suitable threshold $\beta$
\citep{Koyejo+14, Yan+18}. 
 Post-shift only has one degree of freedom, which limits it, but also 
 enables choosing $\beta$ to directly optimize the \textit{true} metric. In contrast, our method acts through surrogate losses to optimize the target metric.  
We argue that post-shift should be a required baseline for experiments on custom metric optimization.

We look forward to seeing further theoretical analysis for handling black-box metrics, and further experimentation comparing methods with fewer but smarter parameters to those with more flexible modeling. 





\bibliographystyle{apalike}
\bibliography{references1,references2,references3}

\clearpage
\newpage
\appendix

\onecolumn
\allowdisplaybreaks
\begin{center}
    \Large \textbf{Optimizing Black-box Metrics with Adaptive Surrogates}\\[10pt]
    \large \textbf{Appendix}\\[20pt]
\end{center}

\textbf{Notations.} We use $[K]$ to denote $\{1, \ldots, K\}$. We use $\|\cdot\|$ to denote the $L_2$-norm. Unless specified otherwise, all smoothness and Lispchitz definitions are with respect to the $L_2$-norm.

\section{Proofs for Lemmas and Theorems}
\label{app:proofs}
\subsection{Proof of Observation~\ref{obs:u-convex}}
\begin{proof}
To see that the vector $[u_1,\cdots,u_K]$ belongs to a convex set, since by assumption $\{\ell_i\}_{i=1}^K$ are convex functions, therefore the set of constraints $\ell_i(\theta)\leq u_i$ defines a convex set in $[\theta,u_1,\cdots,u_K]$ as intersection of sublevel sets of convex functions are convex. 
\end{proof}


\subsection{Proof of Lemma~\ref{lem:u-project}}
\begin{replemma}{lem:u-project}[Restated]
Let $\bu^+$ be the exact projection of $\tilde{\bu}^{t+1} \in \R^K_+$ onto $\cU$. For any solution $\bu^{t+1}$ to \eqref{eq:project}, we have
$\bu^{t+1} \in \cU$, $\bu^{t+1} \leq \bu^+$, and for a monotonic $\psi$,  $\psi(\bu^{t+1}) \leq \psi(\bu^+)$. 
\end{replemma}
We first show how one can compute an exact projection onto $\cU$, and show that the projection described in Lemma \ref{lem:u-project} implements this approximately.
 \begin{lemma}[Exact projection]
The projection $\bu^+$ of $\tilde{\bu}^{t+1} \in \R^K_+$ onto $\cU$ 
is given by:
 \begin{enumerate}[(i)]
     \item $\theta^{t+1} \in \argmin{\theta \in \thetaspace}\, \frac{1}{2}\|\big(\bell(\theta) \,-\, \tilde{\bu}^{t+1}\big)_+\|^2;~~~\bu^{t+1} = \bell(\theta^{t+1})$
     \item $u_k^+ = \max\{\tilde{u}_k,\, u_k^{t+1}\},~\forall k \in [K]$,
 \end{enumerate}
  where $(z)_+ = \max\{0, z\}$, applied element-wise.
  \label{lem:u-exact-projection}
 \end{lemma}
\begin{proof}
It is easy to see that step (i) is a convex problem because each $\ell_k$ is convex in $\theta$, and both $(\cdot)_+$ and $\|\cdot\|^2$ are convex and monotonic in their arguments, making the composition $\|\big(\bell(\theta) \,-\, \tilde{\bu}\big)_+\|^2$ also convex in $\theta$. 

To perform the projection, since step (i) above is a convex problem, the optimality condition gives 
\[\sum_{i=1}^K (\ell_i(\theta^+)-\tilde{u}_i)_+\cdot \mathbbm{1}\{\ell_i(\theta^+)-\tilde{u}_i > 0\}\cdot \nabla_\theta \ell_i(\theta^+) = \boldO_d\]
which is the same as
\begin{equation}
    \sum_{i=1}^K (u_i^+-\tilde{u}_i)\cdot \nabla_\theta \ell_i(\theta^+) = \boldO_d
\label{eq:u-opt}
\end{equation}
by the second step of the procedure. We shall use \eqref{eq:u-opt} to show that $u_i^+$ is the projection in the $\mathcal{U}$-space.

The projection in the $\mathcal{U}$-space can equivalently be written as the following convex problem
\begin{equation*}
\begin{aligned}
& \underset{u_1,\cdots,u_K, \theta}{\text{minimize}}
& & \frac{1}{2} \sum_{i=1}^K (u_i-\tilde{u}_i)^2 \\
& \text{subject to}
& & u_i-\ell_i(\theta) \geq 0\quad \forall i\in [K]\, .
\end{aligned}
\end{equation*}
Introducing the dual variable $\lambda \in \mathbb{R}^K$ and the KKT condition of the problem becomes
\[\sum_{i=1}^K (u_i-\tilde{u}_i)-\sum_{i=1}^K \lambda_i = 0\quad\quad \sum_{i=1}^K \lambda_i\cdot\nabla_\theta\ell_i(\theta) = 0\]
\[u_i-\ell_i(\theta) \geq 0\quad\quad \lambda_i \geq 0 \quad \quad \lambda_i(u_i-\ell_i(\theta)) = 0 \quad \forall i\in[K]\]
if $(u_1,\cdots,u_K,\theta)$ and $\lambda$ are optimal. 

Taking $\lambda_i = u_i^+-\tilde{u}_i$ and $\theta = \theta^+$ with $u_i = u_i^+$, one can easily verify using \eqref{eq:u-opt} that all the conditions hold. Since the optimization problem satisfies Slater's constraint qualification and therefore we can conclude that the primal optimal solution is $\mathbf{u}^+$, as defined in the lemma statement.
\end{proof}

We go on to prove Lemma \ref{lem:u-project}.

\begin{proof}[Proof of Lemma \ref{lem:u-project}]
Because $\bu^{t+1}$ is the surrogate loss at $\theta^{t+1}$, it clearly lies in $\cL$ and hence in the superset $\cU \supseteq \cL$. Next, notice that the over-constrained projection $\bu^{t+1}$ in Lemma \ref{lem:u-project} is the same as step (i) in the exact projection in Lemma \ref{lem:u-project}, with step (ii) giving us  that the exact projection $u_k^+ = \max\{\tilde{u}_k,\, u_k^{t+1}\},~\forall k \in [K]$.
It follows that: $u_{k}^{t+1} \leq u^{+}_{k}, \forall k \in [K].$ So for a monotonic $\psi$, we have $\psi(\bu^{t+1}) \leq \psi(\bu^+)$.
\end{proof}

\subsection{Proof of Theorem \ref{thm:meta-result}}
\begin{reptheorem}{thm:meta-result}[Restated]
Let $M(\theta) = \psi(\bell(\theta)) + \epsilon(\theta)$, for a $\psi$ that is monotonic, $\beta$-smooth and $L$-Lipschitz, and the worst-case slack $\max_{\theta \in \R^d}|\epsilon(\theta)|$ is the minimum among all such decompositions of $M$.

Suppose each $\ell_k$ is $\gamma$-smooth and $\Phi$-Lipschitz in $\theta$ with $\|\bell(\theta)\|\leq G, \, \forall \theta$. Suppose the gradient estimates $\hbw^t$
  satisfy $\E\left[\|\hbw^t \,-\, \nabla\psi(\bell(\theta^t))\|^2\right] \leq \kappa_{\epsilon}, ~\forall t \in [T]$ and the projection step satisfies
$\|(\bell(\theta^{t+1}) - \tilde{\bu}^t)_+\|^2 \leq
\min_{\theta \in \thetaspace}\|(\bell(\theta)- \tilde{\bu}^t)_+\|^2 \,+\, \mathcal{O}(\frac{1}{\beta^2 T}), ~\forall t \in [T]$. Set stepsize $\eta = \frac{1}{\beta^2}$.

Then  Algorithm \ref{algo:pgd} converges to an approximate stationary point of $\psi(\bell(\cdot))$:
 \vspace{-3pt}
 \begin{align*}
 \min_{1\leq t\leq T}&\E\left[\|\nabla \psi(\bell(\theta^t))\|^2\right] \leq
C\bigg(
 \frac{\beta}{\sqrt{T}} + \sqrt{\kappa_\epsilon} + \sqrt{L}\kappa_\epsilon^{1/4} 
 \bigg),
  \vspace{-3pt}
\end{align*}
where the expectation is over the randomness in the gradient estimates, and 
$C = \cO\big(KL\big(\gamma\big(G+\frac{L}{\beta^2}\big)+\Phi^2\big)\big)$.
\end{reptheorem}

While the above theorem prescribes a specific  learning rate $\eta$ for the projected gradient descent, in our experiments, we tune $\eta$ using a held-out validation set.

The proof proceeds in two parts. 
In Section \ref{app-sub:converge_in_u}, we first show that the algorithm converges to an approximate stationary point of $\psi$ over $\cU$. In Section \ref{app-sub:u_to_theta}, we then translate this a guarantee in $\theta$, i.e.\ we show that the algorithm converges to an approximate stationary point of $\psi(\bell(\cdot))$ over $\theta$. 

\subsubsection{Convergence in $\cU$-space}
\label{app-sub:converge_in_u}
\begin{lemma}
\label{lem:converge_in_u}
 Define the gradient mapping at $\bu  \in \cU$ for a vector $g \in \R^K$ as $P(\bu,\,g) := \frac{1}{\eta}(\bu-\Pi_{\cU}(\bu \,-\, \eta \cdot g))$, where $\Pi_{\cU}(z)$ denotes the projection of $z$ onto $\cU$. Then under the assumptions of Theorem \ref{thm:meta-result}, 
\begin{eqnarray*}
\min_{1\leq t\leq T}\E\left[\|{P}(\bu^t, \nabla \psi(\bu^t))\|^2\right]~\leq~
    \cO\left(\frac{\beta^2}{T} \,+\, \kappa_\epsilon  \,+\, L\sqrt{\kappa_\epsilon} 
    \right).
\end{eqnarray*}
\end{lemma}
 
Before we prove this result, we will find it useful to state the following lemma.
\begin{lemma}[Properties of inexact projection]
 \label{lem:inexact_projection}
 Fix $\bu \in \cU$ where $\cU$ is a convex set and
 arbitrary vectors $g_1, g_2 \in\R^K$. Let 
 \[\theta^+_1 \in \argmin{\theta \in \thetaspace}\; \frac{1}{2}\|(\bell(\theta) \,-\, (\bu \,-\, \eta g_1))_+\|^2
 \quad \text{and} \quad  
 \theta^+_2 \in \argmin{\theta \in \thetaspace}\; \frac{1}{2}\|(\bell(\theta) \,-\, (\bu \,-\, \eta g_2))_+\|^2,\]
 and let ${\bu}_1^+ = \max\{\bell(\theta^+_1),\, \bu - \eta g_1\}$ and 
${\bu}_2^+ = \max\{\bell(\theta^+_2),\, \bu - \eta g_2\}$. Define the gradient mapping $P(\bu,\,g_1) := \frac{1}{\eta}(\bu-\bu_1^+)$ and ${P}(\bu,\,g_2) := \frac{1}{\eta}(\bu-{\bu}_2^+)$.
Denote $\tilde{\theta}^+_1, \tilde{\theta}^+_2$ as approximate minimizers such that 
 \begin{equation}
 \frac{1}{2}\|(\bell(\tilde{\theta}^+_1) \,-\, (\bu - \eta g_1))_+\|^2 \,\leq\,
\frac{1}{2}\|(\bell(\theta^+_1) \,-\, (\bu - \eta g_1))_+\|^2\,+\, \alpha
\label{eq:theta-tilde-1}
\end{equation}
 and
 \begin{equation}
\frac{1}{2} \|(\bell(\tilde{\theta}^+_2) \,-\, (\bu - \eta g_2))_+\|^2 \,\leq\,
\frac{1}{2}\|(\bell(\theta^+_2) \,-\, (\bu - \eta g_2))_+\|^2\,+\, \alpha,
\label{eq:theta-tilde-2}
\end{equation}
and let $\widetilde{\bu}_1^+ = \max\{\bell(\tilde{\theta}^+_1),\, \bu - \eta g_1\}$ and 
$\widetilde{\bu}_2^+ = \max\{\bell(\tilde{\theta}^+_2),\, \bu - \eta g_2\}$.
  Define the corresponding gradient mapping 
  $\tilde{P}(\bu,\,g_1) := \frac{1}{\eta}(\bu-\widetilde{\bu}_1^+)$ and $\tilde{P}(\bu,\,g_2) := \frac{1}{\eta}(\bu-\widetilde{\bu}_2^+)$. 
  Then the following holds:
 \begin{enumerate}
     \item  $\|\tilde{P}(\bu,g_1) \,-\, P(\bu,g_1)\| \,\leq\, \frac{\sqrt{2\alpha}}{\eta}$. 
     \item $\langle g_1,\tilde{P}(\bu,g_1)\rangle \geq \frac{3}{4}\|\tilde{P}(\bu,g_1)\|^2-\frac{2\alpha}{\eta^2}$. 
     \item  $\|\tilde{P}(\bu,\,g_1 )\| \,\leq\,\|g_1\| \,+\, \frac{\sqrt{2\alpha}}{\eta} $.
     \item $\|P(\bu,g_1)-P(\bu,g_2)\| \leq \|g_1-g_2\|$. 
     \item $\|\tilde{P}(\bu,g_1)-\tilde{P}(\bu,g_2)\|\leq \|g_1-g_2\|\,+\,2\frac{\sqrt{2\alpha}}{\eta}$.     
 \end{enumerate}
 \end{lemma}
 \begin{proof}
We have:
\begin{eqnarray*}
\frac{1}{2}\|\tilde{\bu}^+_1 \,-\,  (\bu - \eta g_1)\|^2
    &=& \frac{1}{2}\|\max\{\bu - \eta g_1, \bell(\tilde{\theta}_{1}^{+})\} \,-\,  (\bu - \eta g_1)\|^2\\
    &=& \frac{1}{2}\|(\bell(\tilde{\theta}_{1}^{+}) \,-\, (\bu - \eta g_1))_+\|^2\\
    &\leq& 
    \frac{1}{2}\|(\bell(\theta^+_1) \,-\, (\bu - \eta g_1))_+\|^2 \,+\, \alpha \quad\quad \text{(Assumption~(\ref{eq:theta-tilde-1}))}\\
    &=& 
    \frac{1}{2}\|\max\{\bu - \eta g_1, \bell(\theta^+_1)\} \,-\,  (\bu - \eta g_1)\|^2 \,+\, \alpha\\
    &=& \frac{1}{2}\|\bu^+_1 \,-\,  (\bu - \eta g_1)\|^2 \,+\, \alpha,
\end{eqnarray*}
which implies that
\begin{equation}
\label{eqn:approx_prox_guarantee}
g_1^{\top}\tilde{\bu}^+_1+\frac{1}{2\eta}\|\tilde{\bu}^+_1-\bu\|^2 -g_1^{\top}\bu_1^+-\frac{1}{2\eta}\|\bu_1^+-\bu\|^2\leq \frac{\alpha}{\eta}\, .
\end{equation}
Part (1) now follows from 
\begin{align*}
\|\tilde{P}(\bu,g_1) \,-\, P(\bu,g_1)\| &=\frac{1}{\eta}\|\bu_1^+-\tilde{\bu}^+_1\|\\
&\leq \frac{\sqrt{2\eta}}{\eta}\sqrt{F_{g_1}(\tilde{\bu}^+_1)-F_{g_1}(\bu_1^+)-\nabla F_{g_1}(\bu_1^+)^{\top}(\tilde{\bu}^+_1-\bu_1^+)}\\
&\leq \frac{\sqrt{2\eta}}{\eta}\sqrt{\frac{\alpha}{\eta}}\leq \frac{\sqrt{2\alpha}}{\eta}\, ,
\end{align*}
where we used $\frac{1}{\eta}$-strong convexity of the objective $F_{g_1}(\bz):=g_1^{\top}\bz+\frac{1}{2\eta}\|\bz-\bu\|^2$ for $\bz,\bu\in\mathcal{U}$ and the fact that $\bu_1^+$ is the exact minimizer over the convex set $\mathcal{U}$, implying $\nabla F_{g_1}(\bu_1^+)^{\top}(\bz-\bu_1^+)\geq 0 \; \forall \bz\in\mathcal{U} $.\\
For part (4), since $\bu_1^+$ and $\bu_2^+$ are optimal points of function $F_{g_1}(\cdot)$ and $F_{g_2}(\cdot)$ over convex set $\mathcal{U}$ respectively, from optimality condition we have
\begin{equation}
\label{eqn:optimality_cond}
\Big(g_1+\frac{1}{\eta}(\bu_1^+-\bu)\Big)^{\top}(\bz-\bu_1^+)\geq 0\quad \text{and}\quad \Big(g_2+\frac{1}{\eta}(\bu_2^+-\bu)\Big)^{\top}(\bz-\bu_2^+)\geq 0 \quad \text{for all}\; \bz\in\mathcal{U}\, .
\end{equation}
Setting $\bz=\bu_2^+$ in the first and $\bz=\bu_1^+$ in the second equation and summing up we have
\[(g_1-g_2)^{\top}(\bu_2^+-\bu_1^+) \geq\frac{1}{\eta}\|\bu_2^+-\bu_1^+\|^2\, .\]
Therefore using Cauchy-Schwarz
\[\|P(\bu,g_1)-P(\bu,g_2)\|=\frac{1}{\eta}\|\bu_2^+-\bu_1^+\|\leq \|g_1-g_2\|\, .\]
Part (5) now follows immediately from part (1) and (4) by
\begin{align*}
\|\tilde{P}(\bu,g_1)-\tilde{P}(\bu,g_2)\| &\leq \|P(\bu,g_1)-P(\bu,g_2)\|+\|\tilde{P}(\bu,g_1)-P(\bu,g_1)+P(\bu,g_2)-\tilde{P}(\bu,g_2)\|\\
&\leq \|g_1-g_2\|+2\|\tilde{P}(\bu,g_1)-P(\bu,g_1)\|\\
&\leq  \|g_1-g_2\|+\frac{2\sqrt{2\alpha}}{\eta}\, .
\end{align*}
To see part (2), we plug in $\bz=\bu$ in the first equation of display \eqref{eqn:optimality_cond}, giving $g_1^{\top}(\bu-\bu_1^+)\geq \frac{1}{\eta}\|\bu-\bu_1^+\|^2$. 
Moreover from equation \eqref{eqn:approx_prox_guarantee} we know
\[g_1^{\top}(\bu^+_1-\tilde{\bu}^+_1) \geq -\frac{\alpha}{\eta}+\frac{1}{2\eta}\|\tilde{\bu}^+_1-\bu\|^2-\frac{1}{2\eta}\|\bu^+_1-\bu\|^2\, .\]
Consequently,
\[g_1^{\top}(\bu-\tilde{\bu}^+_1)=g_1^{\top}(\bu-\bu^+_1)+g_1^{\top}(\bu^+_1-\tilde{\bu}^+_1) \geq \frac{1}{\eta}\|\bu-\bu_1^+\|^2-\frac{\alpha}{\eta}+\frac{1}{2\eta}\|\tilde{\bu}^+_1-\bu\|^2-\frac{1}{2\eta}\|\bu^+_1-\bu\|^2\, .\]
Now to relate $\|\bu-\bu^+_1\|$ to $\|\bu^+_1-\tilde{\bu}^+_1\|$, we have
\begin{align*}
\frac{1}{2\eta}\|\bu-\tilde{\bu}^+_1\|^2 &\leq \frac{1}{\eta}\|\bu-\bu_1^+\|^2+\frac{1}{\eta}\|\bu^+_1-\tilde{\bu}^+_1\|^2\\
&\leq \frac{1}{\eta}\|\bu-\bu_1^+\|^2+2[F_{g_1}(\tilde{\bu}^+_1)-F_{g_1}(\bu_1^+)-\nabla F_{g_1}(\bu_1^+)^{\top}(\tilde{\bu}^+_1-\bu_1^+)]\\
&\leq \frac{1}{\eta}\|\bu-\bu_1^+\|^2+\frac{2\alpha}{\eta}\, .
\end{align*}
Putting things together $g_1^{\top}\tilde{P}(\bu,g_1) = \frac{1}{\eta}g_1^{\top}(\bu-\tilde{\bu}_1^+) \geq \frac{3}{4}\|\tilde{P}(\bu,g_1)\|^2-\frac{2\alpha}{\eta^2}$, as claimed.\\
Finally, for part (3) since $\|g_1\|\cdot \|\bu-\bu_1^+\|\geq g_1^{\top}(\bu-\bu_1^+)\geq \frac{1}{\eta}\|\bu-\bu_1^+\|^2$ and using part (2),
\begin{align*}
\|\tilde{P}(\bu,\,g_1 )\| = \frac{1}{\eta}\|\bu-\tilde{\bu}_1^+\|& \leq \frac{1}{\eta}\|\bu-\bu_1^+\|+\frac{1}{\eta}\|\bu_1^+-\tilde{\bu}_1^+\|\\
&\leq \|g_1\|+\frac{\sqrt{2\alpha}}{\eta}\, ,
\end{align*}
where we used part (1) for the last step. This concludes the proof of the lemma.
 \end{proof}
 
Equipped with the above results, we move on to prove Lemma \ref{lem:converge_in_u}, i.e.\ to show that the algorithm converges to an approximate stationary point of $\psi$ over $\cU$.
\begin{proof}[Proof of Lemma \ref{lem:converge_in_u}]
We will assume that the gradient estimates $\bw^t$ satsify $\E\left[\|\bw^t \,-\, \nabla\psi(\bell(\theta^t))\|^2\right] \leq \kappa_\epsilon, ~\forall t \in [T]$ and the projection step satisfies
$\frac{1}{2}\|(\bell(\theta^{t+1}) \,-\, \tilde{\bu}^t)_+\|^2 \,\leq\,
\min_{\theta \in \thetaspace}\, \frac{1}{2}\|(\bell(\theta) \,-\, \tilde{\bu}^t)_+\|^2 \,+\, \alpha, ~\forall t \in [T]$. 

 Let $\bu^{t+1} = \bell(\theta^{t+1})$ and $\tilde{\bu}^{t+1} = \max\{\bu^{t+1},\,  \bu^t - \eta\bw^t\}$ be the next iterate had we executed step (ii) of the projection given Lemma \ref{lem:u-project}. 

 Define $\delta^t := \bw^t - \nabla \psi(\bu^t)$. For any $g \in \R^K$, let the gradient mapping $P(\bu, g)$ and approximate gradient mapping $\tilde{P}(\bu, g)$ be defined as in Lemma \ref{lem:inexact_projection}. Note that $\tilde{\bu}^{t+1} = \bu^t - \eta\tilde{P}(\bu^t, \bw^t)$. 
\begin{eqnarray*}
\psi(\bu^{t+1}) 
    &\leq& \psi(\tilde{\bu}^{t+1})
    ~~~\text{(from monotonicity of $\psi$)}
    \\
    &\leq& \psi(\bu^t) \,-\, \eta\,\langle \nabla \psi(\bu^t),\, \tilde{P}(\bu^t, \bw^t) \rangle
    \,+\, \frac{\beta^2}{2}\eta^2 \|\tilde{P}(\bu^t, \bw^t)\|^2
    ~~~\text{(using smoothness of $\psi$)}
    \\
    &=& \psi(\bu^t) \,-\, \eta\,\langle \bw^t,\, \tilde{P}(\bu^t, \bw^t) \rangle
    \,+\, \eta\,\langle \bw^t \,-\, \nabla \psi(\bu^t),\, \tilde{P}(\bu^t, \bw^t) \rangle
    \,+\, \frac{\beta^2}{2}\eta^2 \|\tilde{P}(\bu^t, \bw^t)\|^2\\
    &=& \psi(\bu^t) \,-\, \eta\,\langle \bw^t,\, \tilde{P}(\bu^t, \bw^t) \rangle
    \,+\, \eta\,\langle \delta^t,\, \tilde{P}(\bu^t, \bw^t) \rangle
    \,+\, \frac{\beta^2}{2}\eta^2 \|\tilde{P}(\bu^t, \bw^t)\|^2\\
    &\leq& \psi(\bu^t) \,-\, \left(\frac{3}{4}\eta\,-\,
    \frac{\beta^2}{2}\eta^2\right)\|\tilde{P}(\bu^t, \bw^t)\|^2
    \,+\, \eta\,\langle \delta^t,\, \tilde{P}(\bu^t, \bw^t)\rangle+\frac{2\alpha}{\eta}
    ~~~\text{(from Lemma \ref{lem:inexact_projection}, statement 2)}
    \\
    &=& \psi(\bu^t) \,-\, \left(\frac{3}{4}\eta\,-\,
    \frac{\beta^2}{2}\eta^2\right)\|\tilde{P}(\bu^t, \bw^t)\|^2
    \,+\, \eta\,\langle \delta^t,\, \tilde{P}(\bu^t, \nabla \psi(\bu^t)) \rangle \,+\, \eta\,\langle \delta^t,\, \tilde{P}(\bu^t, \bw^t) \,-\, \tilde{P}(\bu^t, \nabla \psi(\bu^t))\rangle+\frac{2\alpha}{\eta}
    \\
    &\leq& \psi(\bu^t) \,-\, \left(\frac{3}{4}\eta\,-\,
    \frac{\beta^2}{2}\eta^2\right)\|\tilde{P}(\bu^t, \bw^t)\|^2
    \,+\, \eta\,\langle \delta^t,\, \tilde{P}(\bu^t, \nabla \psi(\bu^t)) \rangle \,+\,
    \eta\|\delta^t\|^2
    \,+\, 2\sqrt{2\alpha}\,\|\delta^t\|+\frac{2\alpha}{\eta}
    \\
    &\leq& \psi(\bu^t) \,-\, \left(\frac{3}{4}\eta\,-\,
    \frac{\beta^2}{2}\eta^2\right)\|\tilde{P}(\bu^t, \bw^t)\|^2
    \,+\, \eta\|\delta^t\|\Big(\|\nabla\psi(\bu^t)\|+\frac{\sqrt{2\alpha}}{\eta}\Big) \,+\,
    \eta\|\delta^t\|^2
    \,+\, 2\sqrt{2\alpha}\,\|\delta^t\|+\frac{2\alpha}{\eta} \\
    &\leq& \psi(\bu^t) \,-\, \left(\frac{3}{4}\eta\,-\,
    \frac{\beta^2}{2}\eta^2\right)\|\tilde{P}(\bu^t, \bw^t)\|^2
    \,+\, (\eta\,L \,+\,
    \sqrt{2\alpha})\|\delta^t\| \,+\,
    \eta\|\delta^t\|^2
    \,+\, 2\sqrt{2\alpha}\,\|\delta^t\|+\frac{2\alpha}{\eta} ,
\end{eqnarray*}
where the third-last inequality uses Lemma \ref{lem:inexact_projection}, statement 5 together with Cauchy-Schwarz and the second-last inequality uses Lemma \ref{lem:inexact_projection}, statement 3, and the fact that $\psi$ is $L$-Lipschitz. 
Summing up over $t = 1, \ldots, T$,
\begin{eqnarray*}
\left(\frac{3}{4}\eta\,-\,
    \frac{\beta^2}{2}\eta^2\right)\sum_{t=1}^T\|\tilde{P}(\bu^t, \bw^t)\|^2
    &\leq& 
    \psi(\bu^{1}) \,-\, \psi(\bu^{t+1}) \,+\, 
    \sum_{t=1}^T\left( (\eta\,L \,+\, 3\sqrt{2\alpha})\|\delta^t\| \,+\,
    \eta\|\delta^t\|^2+\frac{2\alpha}{\eta}\right).
\end{eqnarray*}
Taking expectations on both sides and using the assumption $0\leq \psi(\bu)\leq 1\; \forall \bu\in\mathcal{U}$,  
\begin{eqnarray*}
\left(\frac{3}{4}\eta\,-\,
    \frac{\beta^2}{2}\eta^2\right)\sum_{t=1}^T\E\left[\|\tilde{P}(\bu^t, \bw^t)\|^2\right]
    &\leq& 
    1\,+\,\sum_{t=1}^T\left((\eta\,L \,+\, 3\sqrt{2\alpha})\E\left[\|\delta^t\|\right] \,+\,
    \eta\E\left[\|\delta^t\|^2\right]+\frac{2\alpha}{\eta}\right) \\
    &\leq& 1  \,+\,  \sum_{t=1}^T\left((\eta\,L \,+\, 3\sqrt{2\alpha})\sqrt{\E\left[\|\delta^t\|^2\right]} \,+\,
    \eta\E\left[\|\delta^t\|^2\right]+\frac{2\alpha}{\eta}\right)\\
    &\leq& 1  \,+\, T\left( (\eta\,L \,+\, 3\sqrt{2\alpha})\sqrt{\kappa_\epsilon} \,+\,
    \eta\kappa_\epsilon+\frac{2\alpha}{\eta}\right),
\end{eqnarray*}
where we used the assumption on the gradient estimate error $\E\left[\|\delta^t\|^2\right]$ in the last step.
Rearranging we have
\begin{eqnarray*}
\frac{1}{T}\sum_{t=1}^T\E\left[\|\tilde{P}(\bu^t, \bw^t)\|^2\right]
    &\leq&  \frac{1/T \,+\,(\eta\,L \,+\, 3\sqrt{2\alpha})\sqrt{\kappa_\epsilon} \,+\, \eta\kappa_\epsilon+\frac{2\alpha}{\eta}}{\frac{3}{4}\eta\,-\,\frac{\beta^2}{2}\eta^2}.
\end{eqnarray*}
Using Lemma \ref{lem:inexact_projection}, statement 1, 
\begin{eqnarray*}
\frac{1}{T}\sum_{t=1}^T\E\left[\|{P}(\bu^t, \bw^t)\|^2\right]
    &\leq&  
    \frac{2}{T}\sum_{t=1}^T\E\left[\|\tilde{P}(\bu^t, \bw^t)\|^2\right]
    \,+\, \frac{2}{T}\sum_{t=1}^T\E\left[\|\tilde{P}(\bu^t, \bw^t) \,-\, {P}(\bu^t, \bw^t)\|^2\right]
    \\
    &\leq&
    \frac{
    2/T \,+\,
    2(\eta\,L \,+\, 3\sqrt{2\alpha})\sqrt{\kappa_\epsilon} \,+\, 2\eta\kappa_\epsilon+\frac{4\alpha}{\eta}}{\frac{3}{4}\eta\,-\,\frac{\beta^2}{2}\eta^2} \,+\, \frac{4\alpha}{\eta^2}.
\end{eqnarray*}
Setting stepsize $\eta =\frac{1}{\beta^2}$: 
\begin{eqnarray*}
\frac{1}{T}\sum_{t=1}^T\E\left[\|{P}(\bu^t, \bw^t)\|^2\right]
    &\leq& \frac{8\beta^2}{T} \,+\, 8L\sqrt{\kappa_\epsilon} \,+\, 8\kappa_\epsilon \,+\,
    24\beta^2\sqrt{2\alpha\kappa_\epsilon} \,+\, 20\alpha\beta^4\, .
\end{eqnarray*}
We can now bound the average gradient map norm across iterations:
\begin{eqnarray*}
\frac{1}{T}\sum_{t=1}^T\E\left[\|{P}(\bu^t, \nabla \psi(\bu^t))\|^2\right]&\leq&
\frac{2}{T}\sum_{t=1}^T\E\left[\|{P}(\bu^t, \bw^t)\|^2\right]
\,+\,
\frac{2}{T}\sum_{t=1}^T\E\left[\|{P}(\bu^t, \nabla \psi(\bu^t)) \,-\, P(\bu^t, \bw^t)\|^2\right]
\\
&\leq&
\frac{2}{T}\sum_{t=1}^T\E\left[\|P(\bu^t, \bw^t)\|^2\right]
\,+\,
\frac{2}{T}\sum_{t=1}^T\E\left[\| \nabla \psi(\bu^t) \,-\, \bw^t\|^2\right]
\\
    &\leq& 
    \frac{16\beta^2}{T} \,+\, 16L\sqrt{\kappa_\epsilon} \,+\, 16\kappa_\epsilon \,+\,
    48\beta^2\sqrt{2\alpha\kappa_\epsilon} \,+\, 40\alpha\beta^4 \,+\, 2\kappa_\epsilon
\end{eqnarray*}
where we used Lemma \ref{lem:inexact_projection}, statement 4 for the second inequality and the assumption on the gradient estimation error for the last inequality. Thus:
\begin{eqnarray*}
\min_{1 \leq t \leq T}\E\left[\|{P}(\bu^t, \nabla \psi(\bu^t))\|^2\right]
&\leq& 
    \frac{16\beta^2}{T} \,+\, 16L\sqrt{\kappa_\epsilon} \,+\, 18\kappa_\epsilon \,+\,
    48\beta^2\sqrt{2\alpha\kappa_\epsilon} \,+\, 40\alpha\beta^4.
\label{eq:min-t}
\end{eqnarray*}
Now picking $\alpha = \frac{1}{\beta^2 T}$ completes the proof.
\end{proof}
 
\subsubsection{Convergence in $\theta$-space}
\label{app-sub:u_to_theta}
We are now ready to prove Theorem \ref{thm:meta-result}. We translate the near-stationarity result in Lemma from $\bu$-space to $\theta$-space.

\begin{proof}[Proof of Theorem \ref{thm:meta-result}]

For a given $T$, let $t^* \in \amin{1 \leq t \leq T}\, \|P(\bu^{t},\nabla\psi(\bu^{t}))\|^2.$ 
Pick iterates $\theta^{t^*}$ and $\theta^{t^*+1}$ of Algorithm \ref{algo:pgd}. The corresponding iterates in the $\cU$-space are $\bu^{t^*} = \bell(\theta^{t^*})$ and $\bu^{t^*+1} = \bell(\theta^{t^*+1})$.

Further, let $\tilde{\bu}^{t^*+1} = \bu^{t^*}-\eta \nabla \psi(\bu^{t^*})$ be the un-projected next iterate, and $\hat{\bu}^{t^*+1} = \bu^{t^*}-\eta \cdot P(\bu^{t^*},\nabla\psi(\bu^{t^*}))$ be the one obtained after an exact projection, both using exact gradient $\nabla \psi(\bu^{t^*})$. 

We start with the assumption that (as promised by Lemma~\ref{lem:converge_in_u}):
\begin{equation*}
\E[\|P(\bu^{t^*},\nabla\psi(\bu^{t^*}))\|^2]\,=\, \frac{1}{\eta^2}\E[\|\bu^{t^*}-\eta \cdot P(\bu^{t^*},\nabla\psi(\bu^{t^*}))-\bu^{t^*}\|^2] \,=\,\frac{1}{\eta^2}\E[\|\hat{\bu}^{t^*+1}-\bu^{t^*}\|^2]
~\leq \epsilon^2
\end{equation*}
or equivalently,
\begin{equation}
\E[\|\hat{\bu}^{t^*+1}-\bu^{t^*}\|^2]
~\leq \eta^2\epsilon^2
\label{eq:assumption}
\end{equation}
and would like to bound the gradient norm of $\psi(\bell(\cdot))$ at $\theta^{t^*}$.

We start by translating \eqref{eq:assumption} to a guarantee in the $\theta$-space. We know that
\begin{equation}
\hat{\bu}^{t^*+1} \in\arg\min_{\bu\in\mathcal{U}}\, \|\bu-\tilde{\bu}^{t^*+1}\|^2.
\label{eq:uhat}
    \end{equation}
Put together \eqref{eq:assumption} and \eqref{eq:uhat}, and take expectation over randomness in $\bu^{t^*}$, 
\begin{eqnarray*}
\E[\|{\bu}^{t^*} \,-\, \tilde{\bu}^{t^*+1}\|^2] &\leq&
\E[\|\hat{\bu}^{t^*+1} \,-\, \tilde{\bu}^{t^*+1}\|^2] \,+\,
\E[\|\hat{\bu}^{t^*+1} \,-\, {\bu}^{t^*}\|^2] \,+\, 2\E[\|\hat{\bu}^{t^*+1} \,-\, \tilde{\bu}^{t^*+1}\| \|\hat{\bu}^{t^*+1} \,-\, {\bu}^{t^*}\|]\\
&\leq&
\E[\|\hat{\bu}^{t^*+1} \,-\, \tilde{\bu}^{t^*+1}\|^2]
\,+\, \eta^2\epsilon^2 
\,+\, 2\eta\epsilon\sqrt{\E[\|\hat{\bu}^{t^*+1} \,-\, \tilde{\bu}^{t^*+1}\|^2]} 
\\
&\leq&
\E[\|\hat{\bu}^{t^*+1} \,-\, \tilde{\bu}^{t^*+1}\|^2]
\,+\, \eta^2\epsilon^2 
\,+\, 2\eta\epsilon \sqrt{\E[\|{\bu^{t^*}} \,-\, \tilde{\bu}^{t^*+1}\|^2]}
\\
&=&
\E[\|\hat{\bu}^{t^*+1} \,-\, \tilde{\bu}^{t^*+1}\|^2]
\,+\, \eta^2\epsilon^2 
\,+\, 2\eta^2\epsilon \sqrt{\E[\|\nabla\psi(\bu^{t^*})\|^2]} ,
\end{eqnarray*}
where we used Cauchy-Schwarz for the second step. 
Using the fact that $\psi$ is $L$-Lipschitz:
\begin{equation}
\E[\|{\bu}^{t^*} \,-\, \tilde{\bu}^{t^*+1}\|^2] ~\leq~
\E[\|\hat{\bu}^{t^*+1} \,-\, \tilde{\bu}^{t^*+1}\|^2] \,+\, \epsilon',
\label{eq:ut-approx}
\end{equation}
where $\epsilon' = \eta^2(\epsilon ^2+ 2L\epsilon)$.

We also know that $\hat{\bu}^{t^*+1}$ can be equivalently obtained by performing an optimization in the $\theta$-space as follows:
\[
\hat{\theta}^{t^*+1} \in\arg\min_{\theta\in \mathbb{R}^d}\, \|\max\{\bell(\theta), \tilde{\bu}^{t^*+1}\}\,-\,\tilde{\bu}^{t^*+1}\|^2
\]
and setting $\hat{\bu}^{t^*+1} \,=\, \max\{\bell(\hat{\theta}^{t^*+1}), \tilde{\bu}^{t^*+1}\}$. So \eqref{eq:ut-approx} translates to the following guarantee in the $\theta$-space:
\begin{equation}
\E[\|\bell(\theta^{t^*})\,-\,\tilde{\bu}^{t^*+1}\|^2] ~\leq~
\E[\min_{\theta\in \mathbb{R}^d}\, \|\max\{\bell(\theta), \tilde{\bu}^{t^*+1}\}\,-\,\tilde{\bu}^{t^*+1}\|^2] \,+\, \epsilon',
\label{eq:theta-guarantee-1}
\end{equation}
where we have used $\bu^{t^*} = \bell(\theta^{t^*})$. Now since
\begin{eqnarray*}
\|\max\{\bell(\theta^{t^*}),\tilde{\bu}^{t^*+1}\} \,-\,\tilde{\bu}^{t^*+1}\|^2
~=~
\|(\bell(\theta^{t^*})\,-\,\tilde{\bu}^{t^*+1})_+\|^2
~\leq~
\|\bell(\theta^{t^*})\,-\,\tilde{\bu}^{t^*+1}\|^2,
\end{eqnarray*}
together with \eqref{eq:theta-guarantee-1} we have
\begin{equation}
\E[\|\max\{\bell(\theta^{t^*}),\tilde{\bu}^{t^*+1}\} \,-\,\tilde{\bu}^{t^*+1}\|^2]
~\leq~
    \E[\min_{\theta\in \mathbb{R}^d}\, \|\max\{\bell(\theta), \tilde{\bu}^{t^*+1}\}\,-\,\tilde{\bu}^{t^*+1}\|^2] \,+\, \epsilon'\, .
\label{eq:theta-guarantee-2}
\end{equation}

Having translated our initial assumption on the gradient mapping to $\theta$-space, we can now provide a guarantee on the gradient of $\psi(\bell(\cdot))$. Let $Q(\theta) \,:=\,\|\max\{\bell(\theta),\tilde{\bu}^{t^*+1}\} \,-\,\tilde{\bu}^{t^*+1}\|^2
\,=\, \|(\bell(\theta) \,-\,\tilde{\bu}^{t^*+1})_+\|^2 $. 

Taking as given that $Q$ is smooth in $\theta$ with smoothness parameter $\omega$ for now, by standard properties of smooth functions, we have for any $\theta'$:
\[
\|\nabla Q(\theta')\|^2 \,\leq\, 2\omega\cdot(Q(\theta') \,-\, \min_{\theta \in \R^d} Q(\theta)).
\]
Using the above property and \eqref{eq:theta-guarantee-2}, taking expectation on both sides, we have:
\[
\E[\|\nabla Q(\theta^{t^*})\|^2] \,\leq\, 2\omega \epsilon',
\]
or equivalently,
\[
\E\bigg[\left\|2\sum_{k=1}^K(\ell_k(\theta^{t^*})-\tilde{u}_k^{t^*+1})_+\nabla_\theta \ell_k(\theta^{t^*})\right\|^2\bigg] \,\leq\, 2\omega \epsilon',
\]
therefore
\[
4\eta^2\E\bigg[\left\|\sum_{k=1}^K(\nabla \psi_k(\bell^{t^*}))_+\nabla_\theta \ell_k(\theta^{t^*})\right\|^2\bigg] \,\leq\, 2\omega \epsilon',
\]
where we use the short-hand $\bell^{t^*} = \bell(\theta^{t^*})$. By monotonicity of $\psi$, the gradient of $\psi$ is always non-negative, and the above becomes:
\[
4\eta^2\E\bigg[\left\|\sum_{k=1}^K\nabla \psi_k(\bell^{t^*})\nabla_\theta \ell_k(\theta^{t^*})\right\|^2\bigg] \,\leq\, 2\omega \epsilon',
\]
and we have:
\[
\E[\|\nabla_\theta \psi(\bell(\theta^{t^*}))\|^2] \,\leq\, \omega \epsilon'/2\eta^2 \,=\, \omega (\epsilon^2  + 2L\epsilon)/2,
\]
as desired.
It remains to justify the smoothness of $Q(\theta)$. For  any $\theta_1,\theta_2 \in \mathbb{R}^d$,
\begin{align*}
    \|&\nabla Q(\theta_1)-\nabla Q(\theta_2)\| \\
    &= \Big\|2\sum_{k=1}^K(\ell_k(\theta_1)-\tilde{u}_k^{t^*+1})_+\cdot \nabla_\theta\ell_k(\theta_1) - 2\sum_{k=1}^K (\ell_k(\theta_2)-\tilde{u}_k^{t^*+1})_+\cdot \nabla_\theta\ell_k(\theta_2)\Big\|\\
    &\leq 2\sum_{k=1}^K \Big\|(\ell_k(\theta_1)-\tilde{u}_k^{t^*+1})_+\cdot (\nabla_\theta \ell_k(\theta_1)-\nabla_\theta\ell_k(\theta_2))\Big\|+\Big\|\big[(\ell_k(\theta_1)-\tilde{u}_k^{t^*+1})_+-(\ell_k(\theta_2)-\tilde{u}_k^{t^*+1})_+\big]\cdot \nabla_\theta\ell_k(\theta_2)\Big\|\\
    &\leq 2\sum_{k=1}^K |\ell_k(\theta_1)-\tilde{u}_k^{t^*+1}|\cdot \gamma \|\theta_1-\theta_2\|+|\ell_k(\theta_1)-\ell_k(\theta_2)|\cdot \|\nabla_\theta \ell_k(\theta_2)\|\\
    &= 2\sum_{k=1}^K |\ell_k(\theta_1)-\ell_k(\theta^{t^*})+\eta\nabla \psi_k(u^{t^*})|\cdot\gamma \|\theta_1-\theta_2\|+\Phi^2 \|\theta_1-\theta_2\|\\
    &\leq 2K\big[(G+\eta L)\cdot \gamma+\Phi^2\big]\cdot  \|\theta_1-\theta_2\| = 2K\Big[(G+\frac{L}{\beta^2})\cdot \gamma+\Phi^2\Big]\cdot \|\theta_1-\theta_2\|
\end{align*}
where we used $\gamma$-smoothness and $\Phi$-lipschitz property of $\ell_k$ and $\|\bell(\theta)\|\leq G$, together with $(a)_+-(b)_+ \leq |a-b|$, therefore $\omega = 2K\big[(G+\frac{L}{\beta^2})\cdot \gamma+\Phi^2\big]$.
%
%
\end{proof}

\subsection{Proof of Lemma \ref{lem:fd}}
Recall from Algorithm \ref{algo:finite-diff} that the finite difference estimate of the gradient of $\psi$ at $\theta'$ is given by:
\[
\hbw = \frac{1}{m}\sum_{j=1}^m\frac{M(\boldf_{\theta'} \,+\, \Delta^j, \by) \,-\, M(\boldf_{\theta'}, \by) }{\sigma}Z^j.
\]
\begin{replemma}{lem:fd}[Restated]
Let $M$ be as defined in Theorem \ref{thm:meta-result} and $|\epsilon(\theta)| \leq \bar{\epsilon}, \forall \theta$.
 Let $\hbw$ be returned by Algorithm \ref{algo:finite-diff} for a given $\theta'$, $m$ perturbations and $\sigma = \frac{\sqrt{\bar{\epsilon}}}{\sqrt{K}\beta^2}$. 
\begin{align*}
\E\left[\|\hbw \,-\, \nabla\psi(\bell(\theta'))\|^2\right] 
\,\leq\,
\cO\left(\frac{L^2K}{m} + \bar{\epsilon}K^2\beta^2\right),
\end{align*}
where the expectation is over the random perturbations.
\end{replemma}
We will find it useful to re-state results from \citet{Nesterov+17}, extended to our setting.
\begin{lemma}
\label{lem:nesterov}
Suppose $\psi$ is $L$-Lipschitz and $\beta$-smooth. Define $\psi_\sigma(\bu) := \E_{Z \sim \mathcal{N}(\0, \mathbf{I}_K)}\left[\psi(\bu \,+\, \sigma Z)\right]$. 
 Let $\hbw_1 = \frac{1}{m}\sum_{j=1}^m\frac{\psi(\bell(\boldf_\theta \,+\, \Delta^j, \by)) \,-\, \psi(\bell(\boldf_\theta, \by))}{\sigma}Z^j$, where $\Delta^j$ is as defined in Algorithm \ref{algo:finite-diff}. Then:
 \begin{enumerate}
     \item $\hbw_1$ is an unbiased estimate of the gradient of $\psi_\sigma$ at $\bell(\theta)$, i.e., 
 $\E[\hbw_1] \,=\, \nabla \psi_\sigma(\bell(\theta))$. 
    \item 
 $\displaystyle\E\left[\|\hbw_1-\E[\hbw_1]\|^2\right] \,\leq\,\frac{\sigma^2\beta^2}{m}(K+6)^3 \,+\, 
    \frac{4L^2}{m}(K+4).$ 
    \item $\displaystyle\|\nabla\psi_\sigma(\bell(\theta)) \,-\, \nabla\psi(\bell(\theta))\| \,\leq\, 
\frac{\sigma\beta^2}{2}(K+3)^{3/2}.$
 \end{enumerate}
\end{lemma}
\begin{proof}
See Eq.\ (21) in Nesterov et al.\ (2017)  for part 1. Theorem 4 of Nesterov et al.\ together with the fact that $\text{Var}(X) \leq \E[X^2]$ implies part 2. See
Lemma 3 of Nesterov et al.\ for part 3.
\end{proof}

\begin{proof}[Proof of Lemma \ref{lem:fd}]
We can write out the gradient estimate as:
\begin{eqnarray*}
\hbw &=& \frac{1}{m}\sum_{j=1}^m\frac{M(\boldf_\theta \,+\, \Delta^j, \by) \,-\, M(\boldf_\theta,\by) }{\sigma}Z^j
\\
&=&\frac{1}{m}\sum_{j=1}^m\frac{\psi(\bell(\boldf_\theta \,+\, \Delta^j, \by)) \,-\, \psi(\bell(\boldf_\theta, \by)) }{\sigma}Z^j
\,+\, \frac{1}{m}\sum_{j=1}^m\frac{\epsilon(\boldf_\theta \,+\, \Delta^j,\,\by) \,-\, \epsilon(\boldf_\theta,\,\by) }{\sigma}Z^j
\\
&=& \frac{1}{m}\sum_{j=1}^m\frac{\psi(\bell(\theta)+\sigma Z^j) \,-\, \psi(\bell(\theta)) }{\sigma}Z^j
\,+\, \frac{1}{m}\sum_{j=1}^m\frac{\epsilon(\boldf_\theta \,+\, \Delta^j,\,\by) \,-\, \epsilon(\boldf_\theta,\,\by) }{\sigma}Z^j\\
&:=& \hbw_1 + \hbw_2,
\end{eqnarray*}
where $\epsilon(\boldf_\theta,\,\by)$ is the unknown slack function in Section \ref{sec:re-formulation}, re-written in terms of the scores $\boldf_\theta$ and labels $\by$.

Let $\psi_\sigma$ be defined as in Lemma \ref{lem:nesterov}. 
Then the gradient estimate error can be expanded as:  
\begin{eqnarray*}
\E\left[\|\hbw \,-\, \nabla\psi(\bell(\theta))\|^2\right] &\leq&
    2\E\left[\|\hbw \,-\, \nabla\psi_\sigma(\bell(\theta))\|^2\right] \,+\,
    2\|\nabla\psi_\sigma(\bell(\theta)) \,-\, \nabla\psi(\bell(\theta))\|^2\\
    &\leq&
    4\E\left[\|\hbw_1 \,-\, \nabla\psi_\sigma(\bell(\theta))\|^2\right] \,+\,
    4\E\left[\|\hbw_2\|^2\right] \,+\,
    2\|\nabla\psi_\sigma(\bell(\theta)) \,-\, \nabla\psi(\bell(\theta))\|^2\\
    &\leq&
   4\E\left[\|\hbw_1 \,-\, \nabla\psi_\sigma(\bell(\theta))\|^2\right]  \,+\,
    \frac{16\bar{\epsilon}^2}{\sigma^2m}\sum_{j=1}^m \E\left[\|Z^j\|^2\right] \,+\,
    2\|\nabla\psi_\sigma(\bell(\theta)) \,-\, \nabla\psi(\bell(\theta))\|^2 \\
    &\leq&
    \frac{4\sigma^2}{m}\beta^2(K+6)^3 \,+\, 
    \frac{16}{m}L^2(K+4) \,+\, 
    \frac{16\bar{\epsilon}^2K}{\sigma^2} \,+\,
    \frac{\sigma^2}{2}\beta^4(K+3)^{3}, 
\end{eqnarray*}
where we used the fact that (1) $\hbw_1$ is an unbiased estimate of $\nabla\psi_\sigma(\bell(\theta))$ (see part 1 of Lemma \ref{lem:nesterov}); (2) the assumption that $|\epsilon(\theta)| \leq \bar{\epsilon}$; (3) $\|a_1+\cdots+a_m\|^2 \leq m(\|a_1\|^2+\cdots+\|a_m\|^2)$, and the last step follows from Parts 2--3 of Lemma \ref{lem:nesterov}. 
%
%

Setting $\sigma = \frac{\sqrt{\bar{\epsilon}}}{\sqrt{K}\beta^2}$ completes the proof.
\end{proof}

\subsection{Proofs and Discussion for Linear Interpolation Gradient Estimates}
\label{app:ls-proof}
\begin{replemma}{thm:ls}[Restated]
Let $M$ be defined as in Theorem \ref{thm:meta-result} and
$|\epsilon(\theta)| \leq \bar{\epsilon}, \forall \theta$. Assume each $\ell_k$ is 
$\Phi$-Lipschitz in $\theta$ w.r.t.\ the $L_\infty$-norm,
and $\|\bell(\theta)\|\leq G\,\, \forall \theta$.
Suppose for a given $\theta'$, $\sigma$ and perturbation count $m$, the expected covariance matrix for the left-hand-side of the linear system $\bH$ is well-conditioned with the smallest singlular value $\lambda_{\min}(\sum_{i=1}^m \E[\mathbf{H}_i\mathbf{H}_i^{\top}]) \geq \mu_{\min}=
\mathcal{O}(m\sigma^2\Phi^2)$.
Then for any 
 $\delta>0$, 
setting 
$\sigma = \frac{G^{1/3}\bar{\epsilon}^{1/3}}{\Phi K^{3/2}\log(d)^{2/3}\beta^{1/3}}$ and
$m = \frac{G^4K^9\log(d)^4\beta^2\log(K/\delta)}{\bar{\epsilon}^2}$, 
Algorithm \ref{algo:ls} returns  w.p.\ $\geq 1 - \delta$ (over draws of random perturbations)  a gradient estimate $\hbw$ that satisfies:
$$\|\hbw \,-\, \nabla\psi(\bell(\theta'))\|^2
\,\leq\,
\tilde{\cO}\left(G^{1/3}\bar{\epsilon}^{1/3}K^3\beta^{2/3}\right)\, .
$$
\end{replemma}

We first discuss the assumptions in Lemma  \ref{thm:ls} in Section \ref{app-sub:ls-assumptions}.
We then provide the proof for the high probability statement in the lemma in Section \ref{app-sub:ls-proof}. We then show how this can be translated to a bound on the expected gradient error via truncation in Section \ref{app-sub:ls-hp}. 

\subsubsection{Assumptions in Lemma \ref{thm:ls}}
\label{app-sub:ls-assumptions}
We discuss example settings where the assumptions in the lemma hold.

\paragraph{Correlation Assumption on  $\bH$.}
One of the key assumptions we make is that the matrix $\mathbf{H}$ is well-conditioned. Recall that $\mathbf{H}$ is a $m \times K$ matrix, where each row corresponds to a perturbation of the surrogates, and contains differences in the $K$ surrogates $\ell_1, \ldots, \ell_K$ at two independent perturbations to the model parameters $\theta$.
We assume that the smallest singular value of $\bH$'s covariance matrix $\sum_{i=1}^m \E[\mathbf{H}_i\mathbf{H}_i^{\top}]$ scales as $m\sigma^2\Phi^2$. This assumption essentially states that the perturbations on the $K$ surrogates are weakly correlated. The scaling factors $\sigma$ and $\Phi$ come from the fact that Gaussian perturbations on the model parameters $\theta$ have standard deviation $\sigma$ and the surrogates $\ell_k$ are $\Phi$-Lipschitz.

 As an example scenario where this assumption holds, consider a ML fairness task where the instances belong to $K$ non-overlapping protected groups. Further, assume that the group membership attribute is included in the feature vector, i.e., the $d$-dimensional feature vector $\bx = [g_1, \ldots, g_K, \tilde{x}_1, \ldots, \tilde{x}_{d-K}]$, where $g_k$ is a Boolean indicating if the instance belongs to group $k$, and $\tilde{x}_1, \ldots, \tilde{x}_{d-K}$ are group-independent features. A natural choice of surrogates for this application would be average losses computed on the $K$ individual groups. For example, with a linear model $\theta$, we could choose $\ell_k$ to be the average squared loss conditioned on examples from group $k$, i.e.,
$\ell_k(\theta) \,=\, \E_{(x,y) | x_k=1}[(\theta^\top x - y)^2]$.

Note that the first $K$ coordinates of the model vector $\theta$ correspond to weights on the $K$ Boolean group attributes. So adding noise $Z_k \in \R$ to the $k$-th coordinate of $\theta$ only affects scores on examples from the $k$-th group (i.e.,\ examples for which $x_k=1$), and hence only perturbs surrogate $\ell_k$. 
Specifically, adding $Z_k \in \R$ to the $k$-th coordinate of $\theta$ would perturb $\ell_k(\theta)$ to   $\ell_k(\theta) + C_kZ_k + Z_k^2$, where $C_k = 2\E_{(x,y) | x_k=1}[\theta^\top x - y]$, and leave the other surrogates $\ell_j, j \ne k$ unchanged. 


Now suppose we add independent $\sigma$-Gaussian noise to only the first $K$ coordinates of $\theta$. The expected covariance matrix as defined in the lemma statement then takes the form: 
\[
  \sum_{i=1}^m \E[\mathbf{H}_i\mathbf{H}_i^{\top}] \,=\, 
  \begin{bmatrix}
  \cO(m(C_1^2\sigma^2 + \sigma^4)) &  0 & \ldots & 0\\
  \vdots & \vdots & \vdots & \vdots\\
  0 & 0 & \ldots & \cO(m(C_K^2\sigma^2 + \sigma^4))
  \end{bmatrix}\,=\, 
  \begin{bmatrix}
  \Omega(m\sigma^2) &  0 & \ldots & 0\\
  \vdots & \vdots & \vdots & \vdots\\
  0 & 0 & \ldots & \Omega(m\sigma^2)
  \end{bmatrix}\, ,
\]
where recall that the $k$-th column of
$\bH$ contains the differences of $\ell_k(\theta)$ at two different $\sigma$-Gaussian perturbations on the first $K$ coordinates of $\theta$, and $C_k$'s are constants that are independent of the random perturbations. 

In the more general case, where we perturb all coordinates of $\theta$, the assumption on $\bH$ would still hold if there exists a subset of coordinates for each surrogate $\ell_k$ that when  perturbed  produce larger changes to $\ell_k$ than to the other surrogates.

\paragraph{Lipschitz Assumption on $\bell(\theta)$} Another key assumption we make is that the surrogates $\ell_k$ are $\Phi$-Lipschitz w.r.t.\ the $L_\infty$-norm. 
This allows us to produce perturbations in the $K$ surrogates by perturbing the model parameters $\theta$, and do so without  a strong dependence on the dimension of $\theta$ in the error bound. 
Note that the choice of the infinity norm results in a mild logarithmic dependence on the dimension $d$ in the bound. When the surrogates the are not 
$L_\infty$-Lipschitz, but are instead Lipschitz w.r.t.\ the $L_2$-norm,
we prescribe perturbing only a small number of $d' \ll d$ coordinates of $\theta$ that are most closely related to the surrogate (such as e.g. the group attribute coordinates in the fairness example above), and this would result in a bound that has a polynomial dependence on $d'$. 
\subsubsection{Proof of Lemma \ref{thm:ls}}
\label{app-sub:ls-proof}

We will make use of the fact that because we perturb the model parameters $\theta$ with Gaussian random noise, the resulting perturbations on the surrogates $\bell$ follow a sub-Gaussian distribution. We first state a few well-known facts about sub-Gaussian random vectors. 
\begin{lemma}[Properties of sub-Gaussian distribution]
~\vspace{-10pt}
\label{lem:useful_fact}
\begin{enumerate}[(i)]
\item Let $(Z_1,\cdots,Z_d)$ be a vector of i.i.d standard gaussian variables and $f:\mathbb{R}^N\rightarrow \mathbb{R}$ be $\Phi$-Lipschitz w.r.t.\ $L_2$-norm. Then the random variable $f(\sigma Z)-\E[f(\sigma Z)]$ is sub-Gaussian with parameter at most $\sigma\Phi$. 
\item Let $Z_1, \cdots, Z_K$ be $K$ (not necessarily independent) sub-Gaussian random variables with parameters at most $\sigma$. Then the random vector $(Z_1,\cdots,Z_K)$ is a sub-Gaussian random vector with parameter $\sigma K$.
\item For a sub-Gaussian random vector $Z\in\mathbb{R}^K$ with parameter at most $\sigma$, we have for any $p \in \mathbb{N}$: 
    \[\left(\E[\|Z \,-\, \E[Z]\|_2^p]\right)^{1/p} \,\leq\, 2\sqrt{2}\sigma\sqrt{K} \sqrt{p}.\]
\end{enumerate}
\end{lemma}
\begin{proof}
For a proof of (1), see e.g.\ \citet{Wainwright19}, Chapter 2. For a proof of (3), see \citet{Jin}. We now prove (2).

For a random vector $(Z_1, \ldots, Z_K)$ where the coordinates $Z_k$'s are $\sigma$-sub-Gaussian and not necessarily independent, we have that for any $v\in\mathbb{S}^{K-1}$ and $\lambda \in \R$, 
\begin{align*}
    \E[\exp(\lambda v^{\top}(Z-\E[Z]))] &= \E\Big[\prod_{k=1}^{K}\exp\Big(\lambda v_k(Z_k-\E[Z_k])\Big)\Big]\\
    &\leq \prod_{k=1}^K \E\Big[\Big(\exp(\lambda v_k(Z_k-\E[Z_k]))\Big)^{K}\Big]^{1/{K}}\\
    &\leq \prod_{k=1}^K \exp\left(\frac{1}{2}\lambda^2K^2\sigma^2\right)^{1/K}  
    = \prod_{k=1}^K\exp\left(\frac{1}{2}\lambda^2\sigma^2K\right) \leq \exp\left(\frac{1}{2}\lambda^2\sigma^2K^2\right),
\end{align*}
where we have used H\"{o}lder's inequality for the second step.
\end{proof}

We can  write the optimization problem in Algorithm~\ref{algo:ls} as solving the following linear system 
\[\begin{bmatrix}
h'_{11}-h''_{11} & \cdots &h'_{1K}-h''_{1K}\\
\vdots & \cdots & \vdots \\
h'_{m1}-h''_{m1} &  \cdots & h'_{mK}-h''_{mK}
\end{bmatrix}\cdot \hbw = \begin{bmatrix}
\psi(\bell(\theta')+\mathbf{h}'_1)-\psi(\bell(\theta')+\mathbf{h}''_1) + \epsilon_{11} - \epsilon_{12}\\
\vdots\\
\psi(\bell(\theta')+\mathbf{h}'_m)-\psi(\bell(\theta')+\mathbf{h}''_m)
+ \epsilon_{m1} - \epsilon_{m1}
\end{bmatrix},\]
and use the resulting $\hbw\in\mathbb{R}^K$ as  the gradient estimate, where 
we denote $\mathbf{h}'_j :=\bell({\theta'}+\sigma Z_1^j)\,-\, \bell(\theta')\in\mathbb{R}^K$ and $\mathbf{h}''_j :=\bell({\theta'}+\sigma Z_2^j) \,-\, \bell(\theta')\in\mathbb{R}^K$ for $j\in[m]$, and
 $\epsilon_{j1} = \epsilon(\theta' + \sigma Z^j_1)$ and $\epsilon_{j2} =\epsilon(\theta' + \sigma Z^j_2)$.
We further denote 
$$\mathbf{L} = [\bell(\theta'); \ldots; \bell(\theta')] \in \R^{m\times K}
$$
$$
\mathbf{H}' =
[\mathbf{h}'_1;\cdots;\mathbf{h}'_m]\in\mathbb{R}^{m\times K},~~~
\mathbf{H}'' =
[\mathbf{h}''_1;\cdots\mathbf{h}''_m] \in \mathbb{R}^{m\times K}
$$ $$\boldsymbol{\epsilon}_1 = [\epsilon_{11}; \ldots; \epsilon_{m1}] \in \R^{m},~~~ \boldsymbol{\epsilon}_2 = [\epsilon_{12}; \ldots; \epsilon_{m2}] \in \R^{m},$$ and equivalently re-write the above linear system as:
\begin{equation}
(\mathbf{H}' - \mathbf{H}'')\cdot\bw \,=\, \psi(\mathbf{L} + \mathbf{H}') \,-\,  \psi(\mathbf{L} + \mathbf{H}'') \,+\, \boldsymbol{\epsilon}_1 \,-\, \boldsymbol{\epsilon}_2,
\label{eq:ls-linear-system}
\end{equation}
where the matrix $\mathbf{H}$ that we defined in the lemma statement is the same as $\mathbf{H}' - \mathbf{H}''$.

Below we state a lemma involving implications of our assumptions on the left-hand-side perturbation matrices $\mathbf{H}'$ and $\mathbf{H}''$.
\begin{lemma}[Properties of $\mathbf{H}'$ and $\mathbf{H}''$]
\label{lem:subgauss_perturb}
Suppose each $\ell_k(\theta)$ is $\Phi$-Lipschitz w.r.t.\ the $L_\infty$-norm and $\|\bell(\theta)\| \leq G, \forall \theta$. 
Then 
each $\mathbf{h}'_i$ and each $\mathbf{h}''_i$ is a sub-Gaussian vector with parameter at most $\sigma\Phi K$. 
The differences $\mathbf{h}'_i-\mathbf{h}''_i$ are also sub-Gaussian random vectors with parameter at most $2\sigma\Phi K$, and have mean zero. Moreover,  $\|\mathbf{h}'_i\|\leq 2G,\,
\|\mathbf{h}''_i\|\leq 2G,\,
\|\mathbf{h}'_i - \mathbf{h}''_i\|\leq 2G,
\,\forall i$.
\end{lemma}

The proof follows directly from Lemma \ref{lem:useful_fact}(i)--(ii) and the fact that {a function $\ell_k$ that is $\Phi$-Lipschitz w.r.t.\ the $L_\infty$-norm is also $\Phi$-Lipschitz w.r.t.\ the $L_2$-norm}.  We also have from the smoothness of $\psi$ that 
\begin{align}
 \big|\psi(\bell(\theta')+\mathbf{h}'_i)-[\psi(\bell(\theta'))+\nabla \psi(\bell(\theta'))^{\top}\mathbf{h}'_i]\big| &\leq \frac{\beta}{2}\|\mathbf{h}'_i\|_2^2.
 \label{eq:psi-smoothness-h}
\end{align}

With this in hand, we are ready to bound the error in the gradient estimate $\hbw$ compared to $\nabla \psi(\bell(\theta))$.
\begin{proof}[Proof of Lemma \ref{thm:ls}]
The least squares estimate for the linear system in \eqref{eq:ls-linear-system} is given by:
\begin{align*}
 \hbw &= \Big((\mathbf{H}'-\mathbf{H}'')^\top(\mathbf{H}'-\mathbf{H}'')\Big)^{-1}(\mathbf{H}'-\mathbf{H}'')^\top [\psi(\bL+\mathbf{H}')+\boldsymbol{\epsilon_1}-\psi(\bL+\mathbf{H}'')-\boldsymbol{\epsilon_2}]\\
    &=\Big((\mathbf{H}'-\mathbf{H}'')^\top(\mathbf{H}'-\mathbf{H}'')\Big)^{-1}(\mathbf{H}'-\mathbf{H}'')^\top\Big[(\mathbf{H}'-\mathbf{H}'')\nabla \psi(\bell)+\psi(\bL+\mathbf{H}')-\psi(\bL+\mathbf{H}'')\\ &\hspace{10cm}-(\mathbf{H}'-\mathbf{H}'')\nabla \psi(\bell)+\boldsymbol{\epsilon_1}-\boldsymbol{\epsilon_2}\Big]\\
    &= \nabla \psi(\bell(\theta'))+\Big((\mathbf{H}'-\mathbf{H}'')^\top(\mathbf{H}'-\mathbf{H}'')\Big)^{-1}(\mathbf{H}'-\mathbf{H}'')^\top\Big[\psi(\bL+\mathbf{H}')-\psi(\bL+\mathbf{H}'')\\ &\hspace{10cm}- (\mathbf{H}'-\mathbf{H}'')\nabla \psi(\bell)+\boldsymbol{\epsilon_1}-\boldsymbol{\epsilon_2}\Big]\, .
\end{align*}
The error in the least squares based gradient estimate is then:
\begin{align}
    \|&\hbw -\nabla \psi(\bell(\theta'))\|\nonumber\\ &\leq
    \underbrace{\Big\|\Big((\mathbf{H}'-\mathbf{H}'')^\top(\mathbf{H}'-\mathbf{H}'')\Big)^{-1}\Big\|_{\text{op}}}_{\text{term}_1}
    \underbrace{    
    \Big\|
    (\mathbf{H}'-\mathbf{H}'')^{\top}\Big[
    \psi(\bL+\mathbf{H}')-\psi(\bL+\mathbf{H}'')-(\mathbf{H}'-\mathbf{H}'')\nabla \psi(\bell)+\boldsymbol{\epsilon_1}-\boldsymbol{\epsilon_2}
    \Big]\Big\|_2}_{\text{term}_2}\, .
    \label{eq:grad-decomposition}
\end{align}

\textbf{Bounding the second term in \eqref{eq:grad-decomposition}}.
 We first bound the second term in \eqref{eq:grad-decomposition}. 
We have:
\begin{align*}
\displaystyle
\lefteqn{
    \|\psi(\bL+\mathbf{H}')-\psi(\bL+\mathbf{H}'')-(\mathbf{H}'-\mathbf{H}'')\nabla \psi(\bell)+\boldsymbol{\epsilon_1}-\boldsymbol{\epsilon_2}\|_2}\\
    &=\Big\|\psi(\bL+\mathbf{H}')-\mathbf{H}'\nabla\psi(\bell)-\psi(\bL)+\psi(\bL)+\mathbf{H}''\nabla \psi(\bell)-\psi(\bL+\mathbf{H}'')+\boldsymbol{\epsilon_1}-\boldsymbol{\epsilon_2}\Big\|_2\\
    &\leq \Big\|\psi(\bL+\mathbf{H}')-\mathbf{H}'\nabla\psi(\bell)-\psi(\bL)\Big\|_2
    +\Big\|\psi(\bL)+\mathbf{H}''\nabla \psi(\bell)-\psi(\bL+\mathbf{H}'')\Big\|_2+\|\boldsymbol{\epsilon_1}\|_2+\|\boldsymbol{\epsilon_2}\|_2\\
    &\leq \frac{\beta}{2}\|\mathbf{H}'\|_{\text{F}}^2+\frac{\beta}{2}\|\mathbf{H}''\|_{\text{F}}^2+2\sqrt{m}\bar{\epsilon} ,
\end{align*}
 where we used
 \eqref{eq:psi-smoothness-h} and the
 the assumption $|\epsilon(\theta)| \leq \bar{\epsilon} \;\forall \theta$. This in turn gives
 \begin{align*}
\text{term}_2 \,=\,
\Big\|&(\mathbf{H}'-\mathbf{H}'')^{\top}\Big[\psi(\bL+\mathbf{H}')-\psi(\bL+\mathbf{H}'')-(\mathbf{H}'-\mathbf{H}'')\nabla \psi(\bell)+\boldsymbol{\epsilon_1}-\boldsymbol{\epsilon_2}\Big]\Big\|_2\\
&\leq \sum_{j=1}^m \|\mathbf{h}'_j-\mathbf{h}''_j\|\cdot \frac{\beta}{2}(\|\mathbf{h}'_j\|_2^2+\|\mathbf{h}''_j\|_2^2)+2\sqrt{m}G\sqrt{m}\bar{\epsilon} 
  \end{align*}
 where each $\mathbf{h}'_j$ is of length $K$ with (correlated) subgaussian coordinates. Therefore using Cauchy-Schwarz,
 \[\E[\|\mathbf{h}'_j-\mathbf{h}''_j\|\cdot \|\mathbf{h}'_j\|_2^2] \leq  \sqrt{\E[\|\mathbf{h}'_j-\mathbf{h}''_j\|_2^2]}\cdot \sqrt{\E[\|\mathbf{h}'_j\|_2^4]}\, .\]
Note that 
$\E[h'_{ij}] \,=\, \E[\ell_j(\theta'+\sigma Z^i)]-\ell_j(\theta') 
\leq \sigma\Phi \E[\|Z^i\|_{\infty}]\leq \cO(\sigma\Phi\sqrt{\log(d)})$,
where we've used that the max of $d$ independent standard normal random variables scales as $\sqrt{\log(d)}$. Similarly, $\E[h''_{ij}] \leq \cO(\sigma\Phi\sqrt{\log(d)})$. 
Together with these facts and Lemma~\ref{lem:subgauss_perturb} and Lemma~\ref{lem:useful_fact}(iii) we have 
 \[
 \sqrt{\E[\|\mathbf{h}'_j\|_2^4]} 
 \leq \sqrt{8(4\sqrt{2}\sigma\Phi K^{3/2} )^4+ \cO(\sigma^2\Phi^2\log(d)K)^{2}} \leq \cO(\sigma^2\Phi^2K^3 (\log(d))^2)
 \]
where we used triangle inequality and $(a+b)^p \leq 2^{p-1} (a^p+b^p)$.
Similarly, we have:
 \[\sqrt{\E[\|\mathbf{h}'_j-\mathbf{h}''_j\|_2^2]} \leq 4\sigma\Phi K^{3/2}\, .\]
Now since $\|\mathbf{h}'_j\|_2\leq G$, we can apply Hoeffding's inequality to these bounded random variables to get 
\begin{align*}
\P \left(\sum_{j=1}^m\|\mathbf{h}'_j-\mathbf{h}''_j\|_2\cdot \|\mathbf{h}'_j\|_2^2  \,\geq\,
\cO(\sigma^3\Phi^3  K^{9/2}(\log(d))^2m)\,+\,mt\right)  \leq 2\exp\left(-\frac{2mt^2}{G^6}\right),
\end{align*}
which further gives us:
\begin{align}
\P \left(\text{term}_2  \,\geq\,
\cO(\sigma^3\Phi^3  K^{9/2}(\log(d))^2m\beta) 
\,+\,m\beta t\,+\, 2mG\bar{\epsilon}\right)
 \leq 2\exp\left(-\frac{2mt^2}{G^6}\right),
\label{eq:ls-second-term}
\end{align}
 
\textbf{Bounding the first term in \eqref{eq:grad-decomposition}}.
Now the first term in \eqref{eq:grad-decomposition} is simply
\[
\text{term}_1 \,=\,
\Big\|\Big((\mathbf{H}'-\mathbf{H}'')^\top(\mathbf{H}'-\mathbf{H}'')\Big)^{-1}\Big\|_{\text{op}}= \lambda_{\min}^{-1}\Big((\mathbf{H}'-\mathbf{H}'')^{\top}(\mathbf{H}'-\mathbf{H}'')\Big)\, .\]
Let us denote $\boldsymbol{\hat{\Sigma}}:= \sum_{i=1}^m (\mathbf{h}'_{i}-\mathbf{h}''_{i})(\mathbf{h}'_{i}-\mathbf{h}''_{i})^{\top}$ as the empirical covariance matrix.
%
 We now apply a matrix Chernoff inequality (see e.g.\ \citet{Tropp_2015}) to lower bound the smallest eigenvalue of $\boldsymbol{\hat{\Sigma}}$. We first note that the largest eigenvalue of this matrix is bounded above:
  \[\lambda_{\max}(\boldsymbol{\hat{\Sigma}}) = \max_{\|u\| = 1} \frac{1}{m}\sum_{i=1}^m \big((\mathbf{h}'_{i}-\mathbf{h}''_{i})^{\top}u\big)^2 \leq 4G^2\, ,\]
  This together with the matrix Chernoff bound gives us for $\mu_{\min} \leq \lambda_{\min}(\boldsymbol{\hat{\Sigma}})$, we have
  \[\P\left(\lambda_{\min}(\boldsymbol{\hat{\Sigma}}) \leq \frac{\mu_{\min}}{2}\right) \leq K\cdot \exp\Big(-\frac{ \mu_{\min}}{32G^2}\Big).\]
 The assumption
 $\mu_{\min} = 
\mathcal{O}(m\sigma^2\Phi^2)$ then yields:
\begin{align}
    \P\left(\text{term}_1 \leq \cO(m\sigma^2\Phi^2)\right) \leq K\cdot \exp\Big(-\frac{ m\sigma^2\Phi^2}{G^2}\Big).
\label{eq:ls-first-term}
\end{align}

Combining the above bound \eqref{eq:ls-first-term} with the bound on the second term \eqref{eq:ls-second-term} (picking $t=\sigma^3\Phi^3$), we get the following tail bound:
  \[\P\left(\|\hbw -\nabla \psi(\bell(\theta'))\| \geq \cO\left(\sigma\Phi K^{9/2}\log(d)^2\beta +\frac{G\bar{\epsilon}}{\sigma^2\Phi^2}\right)\right) \leq K\cdot \exp\Big(-\frac{m\sigma^2\Phi^2 }{G^2}\Big)+4\exp\Big(-\frac{2m\sigma^6\Phi^6}{G^6}\Big)\, .\]

Then  
for any 
$\delta>0$, setting $
\displaystyle\sigma = \frac{G^{1/3}\bar{\epsilon}^{1/3}}{\Phi K^{3/2}\log(d)^{2/3}\beta^{1/3}}$ and
$\displaystyle m = \frac{G^4K^9\log(d)^4\beta^2\log(K/\delta)}{\bar{\epsilon}^2}$,  Algorithm \ref{algo:ls} returns w.p.\ $\geq 1 - \delta$ (over draws of random perturbations) a gradient estimate $\hbw$ that satisfies:
$$\|\hbw \,-\, \nabla\psi(\bell(\theta'))\|^2
\,\leq\,
\cO\left(G^{1/3}\bar{\epsilon}^{1/3}K^3(\log(d))^{4/3}\beta^{2/3}\right),
$$
which completes the proof.
\end{proof}

\subsubsection{Translating to a Bound on the Expected Error}
\label{app-sub:ls-hp}
Lemma \ref{thm:ls} provides a high probability bound on the gradient estimation error. This means that with a small probability the gradient estimation error may not be bounded. To translate this high probability bound into a bound on the expected gradient error, we first truncate the estimated gradients to be in a bounded range:
\[
\text{trunc}(\bw) \,=\, 
\begin{cases}
\bw& \text{if}~ \|\bw\| \leq 2\sqrt{K}L\\
\0 & \text{otherwise}
\end{cases},
\]
where $L$ is the Lipschitz constant for $\psi$. 
\begin{corollary}
Under the assumptions in Lemma \ref{thm:ls},  for any 
 $\delta \in (0,1)$, 
 setting  $
\sigma = \frac{G^{1/3}\bar{\epsilon}^{1/3}}{\Phi K^{3/2}\log(d)^{2/3}\beta^{1/3}}$ and
$m = \frac{G^4K^9\log(d)^4\beta^2\log(K/\delta)}{\bar{\epsilon}^2}$, 
Algorithm \ref{algo:ls} returns  a gradient estimate $\hbw$ that satisfies:
\[
\E\left[\|\truncg \,-\, \nabla\psi(\bell(\theta'))\|^2\right] ~\leq~ \tilde{\cO}\left(G^{1/3}\bar{\epsilon}^{1/3}K^3\beta^{2/3}\right)\,+\, 10KL^2\delta.
\]
\end{corollary}
\begin{proof}
Because both the truncated gradient estimates and the true gradients are bounded, the gradient error is trivially bounded by:
\begin{equation}
    \|\truncg \,-\, \nabla\psi(\bell(\theta'))\|^2
    \,\leq\,
    2(\|\truncg\|^2 \,+\, \|\nabla\psi(\bell(\theta'))\|^2)
    \,\leq\,
    2(4KL^2 \,+\, L^2)
    \,\leq\, 10KL^2.
    \label{eq:ghat_trivial}
\end{equation}

In the case where $\|\bw\| \leq 2\sqrt{K}L$, the gradient error for the truncated $\bw$ is the same as that for $\bw$: 
\begin{equation}
    \|\truncg \,-\, \nabla\psi(\bell(\theta'))\|^2
    ~=~ \|\bw \,-\, \nabla\psi(\bell(\theta'))\|^2.
    \label{eq:ghat_le}
\end{equation}
When $\|\bw\| > 2\sqrt{K}L$, the gradient error for the truncated estimates $\truncg$ is upper bounded by:
\[
\|\truncg \,-\, \nabla\psi(\bell(\theta'))\|^2 \,=\, 
\|\nabla\psi(\bell(\theta'))\|^2 \,\leq\, L^2,
\]
whereas the the gradient error for the original estimates $\bw$ is lower bounded by:
\begin{eqnarray*}
\|\bw\,-\, \nabla\psi(\bell(\theta'))\|^2 &\geq& \max_{k \in [K]}\, \left(\hat{g}_k \,-\, \nabla_k\psi(\bell(\theta'))\right)^2
\,\geq\,
 \left(\max_{k \in [K]}|\hat{g}_k| \,-\, \max_{k \in [K]}|\nabla_k\psi(\bell(\theta'))|\right)^2\\ &\geq&
\left(\frac{1}{\sqrt{K}}(2\sqrt{K}L) \,-\, L\right)^2
\,=\, L^2\, .
\end{eqnarray*}
Therefore even in this case, the gradient error for 
$\truncg$ is bounded by that for $\bw$:
\begin{equation}
    \|\truncg \,-\, \nabla\psi(\bell(\theta'))\|^2 \,\leq\, L^2 \,\leq\,
    \|\bw \,-\, \nabla\psi(\bell(\theta'))\|^2.
    \label{eq:ghat_ge}
\end{equation}


Combining \eqref{eq:ghat_le} and \eqref{eq:ghat_ge} with the trivial upper bound in \eqref{eq:ghat_trivial}
allows us to convert the
the high probability result in Lemma \ref{thm:ls} to the following bound on the expected error. For any 
$\delta\in(0,1)$, setting $
\sigma = \frac{G^{1/3}\bar{\epsilon}^{1/3}}{\Phi K^{3/2}\log(d)^{2/3}\beta^{1/3}}$ and
$m = \frac{G^4K^9\log(d)^4\beta^2\log(K/\delta)}{\bar{\epsilon}^2}$, we have: 
\[
\E\left[\|\truncg \,-\, \nabla\psi(\bell(\theta'))\|^2\right] ~\leq~ \tilde{\cO}\left((1-\delta)G^{1/3}\bar{\epsilon}^{1/3}K^3\beta^{2/3}\right)\,+\, 10\delta KL^2,
\]
as desired.
\end{proof}

\if 
For a matrix $\mathbf{H}\in \mathbb{R}^{m\times K}$ with $m>K$ and independent rows $\mathbf{H}_i$ whose entries are zero-mean, sub-gaussian random variables with parameter $2\sigma\Phi$, suppose the second moment matrix $\boldsymbol\Sigma = \E[\mathbf{H}_i\mathbf{H}_i^
    {\top}]$ satisfies $\sum_{j\neq i}\boldsymbol\Sigma_{ij} \leq \boldsymbol\Sigma_{ii} = \mathcal{O}(\sigma^2\Phi^2) \; \forall i\in[K]$, then the smallest eigenvalue $\lambda_{\min}(\boldsymbol\Sigma)$ is bounded away from 0 from Gershgorin circle theorem. We make the assumption that the expected covariance matrix has $\lambda_{\min}$ bounded below by $\mu_{\min}$, i.e., $\lambda_{\min}(\sum_{i=1}^m \E[\mathbf{H}_i\mathbf{H}_i^{\top}]) \geq \mu_{\min}=\mathcal{O}(m\sigma^2\Phi^2)$. This essentially says (1) the subgaussian parameter $\sigma\Phi$ (and therefore the variance) is large enough; (2) the coordinates are only weakly correlated such that the covariance matrix is diagonally dominate. 
\fi

\section{Handling Non-smooth Metrics} \label{app:non-smooth}
For $\psi$ that is only $L$-Lipschitz and non-smooth, we extend the finite difference gradient estimate in Section \ref{sec:fd} with a two-step perturbation method, as detailed in Algorithm \ref{algo:finite-diff-2-step}. This approach can be seen as computing a finite-difference gradient estimate for a smooth approximation to the original $\psi$, given by
$\psi_{\sigma_1}(\bu) := \E\left[\psi(\bu \,+\, \sigma_1 Z_1)\right]$, where $Z_1 \sim \mathcal{N}(\0, \mathbf{I}_K)$. Since $\psi_{\sigma_1}$ is a convolution of $\psi$ with a Gaussian density kernel, it is always smooth. For this setting, we build on recent work by  \citet{Duchi}, and 
show that the two-step perturbation approach provides a gradient estimate for $\psi_{\sigma_1}$. 

\begin{lemma}[Two-step finite difference gradient estimate]
\label{lem:fd-2-step}
Let $M(\theta) = \psi(\bell(\theta)) + \epsilon(\theta)$, for a $\psi$ that is $L$-Lipschitz, and the worst-case slack $\max_{\theta \in \R^d}|\epsilon(\theta)|$ is the minimum among all such decompositions of $M$. Suppose $|\epsilon(\theta)| \leq \bar{\epsilon},\,\forall \theta$.
Let $\hbw$ be returned by Algorithm \ref{algo:finite-diff-2-step} for 
a fixed $\sigma_1 > 0$ and $\sigma_2 = \sqrt{\frac{\sigma_1}{  K^{3/2}L}}$.
Then:
\[
\E\left[\|\hbw \,-\, \nabla\psi_{\sigma_1}(\bell(\theta))\|^2\right] ~\leq~
\tilde{\cO}\left(
\frac{L^{7/4}K^{13/8}}{m\sigma_{1}^{1/4}} \,+\, 
\frac{LK^{5/2}\bar{\epsilon}^2}{\sigma_1}
\right).
\]
\end{lemma}

\begin{algorithm}[t]
\caption{Two-step Finite-difference Gradient Estimate}
\label{algo:finite-diff-2-step}
    \begin{algorithmic}[1]
   \STATE \textbf{Input:} $\theta \in \R^d, M, \ell_1,\cdots,\ell_k$, estimation accuracy $\epsilon$
   \STATE Draw $Z_1^1, \ldots, Z_1^m, Z_2^1, \ldots, Z_2^m \sim \mathcal{N}(\0, \mathbf{I}_K)$
   \STATE Find $\Delta_1^j \in \R^n$ s.t.\ $\bell(\boldf_\theta \,+\, \Delta_1^j, \by) \,=\, \bell(\boldf_\theta, \by) \,+\, \sigma_1 Z_1^j$ for $j=1, \ldots, m$
   \STATE Find $\Delta_2^j \in \R^n$ s.t.\ $\bell(\boldf_\theta \,+\, \Delta_2^j, \by) \,=\, \bell(\boldf_\theta, \by) \,+\, \sigma_1 Z_1^j \,+\, \sigma_2 Z_2^j$ for $j=1, \ldots, m$
   \STATE $\displaystyle \hbw = \frac{1}{m}\sum_{j=1}^m\frac{M(\boldf_\theta \,+\, \Delta_2^j,\, \by) \,-\, M(\boldf_\theta \,+\, \Delta_1^j,\,\by) }{\sigma_2}Z_2^j$
   \STATE \textbf{Output}:  $\hbw$ 
 \end{algorithmic}
\end{algorithm}

Drawing upon the result of Theorem~\ref{thm:meta-result}, we can repeat the analysis on the smooth function $\psi_{\sigma_1}(\cdot)$ to get the following convergence guarantee for Algorithm \ref{algo:pgd}.
\begin{corollary}[Convergence of Algorithm \ref{algo:pgd} for non-smooth $\psi$]
\label{cor:non-smooth} 
Let $M(\theta) = \psi(\bell(\theta)) + \epsilon(\theta)$, for a $\psi$ that is monotonic, and $L$-Lipschitz, and the worst-case slack $\max_{\theta \in \R^d}|\epsilon(\theta)|$ is the minimum among all such decompositions of $M$.

Suppose each $\ell_k$ is $\gamma$-smooth and $\Phi$-Lipschitz in $\theta$ with $\|\bell(\theta)\|\leq G, \, \forall \theta$. Suppose the gradient $\hbw^t$ are estimated with Algorithm \ref{algo:pgd} for a choice $\sigma_1 > 0$, number of perturbation $m$, and $\sigma_2 = \sqrt{\frac{\sigma_1}{K^{3/2}L}}$. Suppose the projection step satisfies
$\|(\bell(\theta^{t+1}) - \tilde{\bu}^t)_+\|^2 \leq
\min_{\theta \in \thetaspace}\|(\bell(\theta)- \tilde{\bu}^t)_+\|^2 \,+\, \mathcal{O}(\frac{\sigma_1^2}{TKL^2}), ~\forall t \in [T]$. Set stepsize $\eta = \frac{\sigma_1^2}{KL^2}$.

Then  Algorithm \ref{algo:pgd} converges to an approximate stationary point of the smooth approximation $\psi_{\sigma_1}(\bell(\cdot))$:
 \begin{align*}
 \min_{1\leq t\leq T}&\E\left[\|\nabla \psi_{\sigma_1}(\bell(\theta^t))\|^2\right] \leq
C\bigg(
 \frac{\sqrt{K}L}{\sigma_1\sqrt{T}} + \sqrt{\kappa} + \sqrt{L}\kappa^{1/4} 
 \bigg),
\end{align*}
where the expectation is over the randomness in the gradient estimates, and 
$C = \cO\big(KL\big(\gamma\big(G+\frac{\sigma_1^2}{KL}\big)+\Phi^2\big)\big)$ and $\kappa = \tilde{\cO}\Big(
\frac{L^{7/4}K^{13/8}}{m\sigma_{1}^{1/4}} \,+\, 
\frac{LK^{5/2}\bar{\epsilon}^2}{\sigma_1}
\Big)$.
\end{corollary}

The above result  guarantees convergence to the stationary point of the smoothed metric $\psi_{\sigma_1}(\bell(\cdot))$ and not the original metric $\psi(\bell(\cdot))$. However, as long as the surrogate functions $\bell$ are continuously differentiable, by taking $\sigma_1\rightarrow 0$ and allowing $T$ to increase as $\sigma_1$ decreases, the algorithm can be made to converge to a stationary point of the original metric $\psi(\bell(\cdot))$, in the sense of Clark-subdifferential  (see e.g.\ \citet{Garmanjani+13}).

\subsection{Proof of Lemma \ref{lem:fd-2-step}}
We will find it useful to re-state results from \citet{Duchi} and \citet{Nesterov+17}, extended to our setting. 
\begin{lemma}
\label{lem:duchi}
Suppose $\psi$ is $L$-Lipschitz.
Define 
$\psi_{\sigma_1}(\bu) := \E_{Z_1 \sim \mathcal{N}(\0, \mathbf{I}_K)}\left[\psi(\bu \,+\, \sigma_1 Z_1)\right]$ and 
$\psi_{\sigma_1, \sigma_2}(\bu) := \E_{Z_2 \sim \mathcal{N}(\0, \mathbf{I}_K)}\left[\psi_{\sigma_1}(\bu \,+\, \sigma_2 Z_2)\right]$. 
 Let $\hbw_1 = \frac{1}{m}\sum_{j=1}^m\frac{\psi(\bell(\boldf_\theta \,+\, \Delta_2^j, \by)) \,-\, \psi(\bell(\boldf_\theta \,+\, \Delta_1^j, \by)) }{\sigma_2}Z_2^j$, where $\Delta_1^j, \Delta_2^j$ are as defined in Algorithm \ref{algo:finite-diff-2-step}. Then:
 \begin{enumerate}
     \item $\hbw_1$ is an unbiased estimate of the gradient of $\psi_{\sigma_1,\sigma_2}$ at $\bell(\theta)$, i.e., 
 $\E[\hbw_1] \,=\, \nabla \psi_{\sigma_1,\sigma_2}(\bell(\theta))$. 
 \item $\psi_{\sigma_1}(\cdot)$ is smooth with smoothness parameter $\displaystyle\frac{\sqrt{K}L}{\sigma_1}$ and Lipschitz with constant $L$.
    \item 
 $\displaystyle 
 \E\left[\|\hbw_1-\E[\hbw_1]\|^2\right] \,\leq\,\frac{CL^2K}{m}\left(\sqrt{\frac{\sigma_2}{\sigma_1}}K\,+\,\log K\,+\, 1\right)$
for some constant $C$.
    \item $\displaystyle\|\nabla\psi_{\sigma_1, \sigma_2}(\bell(\theta)) \,-\, \nabla\psi_{\sigma_1}(\bell(\theta))\| \,\leq\, 
\frac{\sigma_2}{2}\frac{\sqrt{K}L}{\sigma_1}(K+3)^\frac{3}{2}.$
 \end{enumerate}
 \begin{proof}
Part 1 follows by trivially observing
\[\E_{Z_1,Z_2}[\hbw_1]= \E_{Z_2}\Big[\frac{\psi_{\sigma_1}(\bu+\sigma_2Z_2)-\psi_{\sigma_1}(\bu)}{\sigma_2}Z_2\Big] = \nabla \psi_{\sigma_1,\sigma_2}(\bu)  \]
where we invoked part 1 of Lemma~\ref{lem:nesterov}. See Lemma 2 of \citet{Nesterov+17}  for part 2. Part 2 together with Lemma 2 in \citet{Duchi} give the result in part 3. See
Lemma 3  of \citet{Nesterov+17} for part 4.
 \end{proof}
\end{lemma}

Now we are ready to bound the MSE in gradient estimate.
\begin{proof}[Proof of Lemma \ref{lem:fd-2-step}]
We can write out the gradient estimate as:
\begin{eqnarray*}
\hbw &=& \frac{1}{m}\sum_{j=1}^m\frac{M(\boldf_\theta \,+\, \Delta_2^j,\,\by) \,-\, M(\boldf_\theta\,+\, \Delta_1^j,\,\by) }{\sigma_2}Z_2^j
\\
&=&\frac{1}{m}\sum_{j=1}^m\frac{\psi(\bell(\boldf_\theta \,+\, \Delta_2^j,\,\by)) \,-\, \psi(\bell(\boldf_\theta \,+\, \Delta_1^j,\,\by))}{\sigma_2}Z_2^j
\,+\, \frac{1}{m}\sum_{j=1}^m\frac{\epsilon(\boldf_\theta \,+\, \Delta_2^j,\,\by) \,-\, \epsilon(\boldf_\theta\,+\, \Delta_1^j,\,\by) }{\sigma_2}Z_2^j
\\
&=& \frac{1}{m}\sum_{j=1}^m\frac{\psi(\bell(\theta)+\sigma_1Z_1^j+\sigma_2Z_2^j) \,-\, \psi(\bell(\theta)+\sigma_1 Z_1^j)}{\sigma_2}Z_2^j
\,+\, \frac{1}{m}\sum_{j=1}^m\frac{\epsilon(\boldf_\theta \,+\, \Delta_2^j,\,\by) \,-\, \epsilon(\boldf_\theta\,+\, \Delta_1^j,\,\by) }{\sigma_2}Z_2^j\\
&:=& \hbw_1 + \hbw_2,
\end{eqnarray*}
where $\epsilon(\boldf_\theta,\,\by)$ is the unknown slack function in Section \ref{sec:re-formulation}, re-written in terms of the scores $\boldf_\theta$ and labels $\by$.

Let $\psi_{\sigma_1}$ and $\psi_{\sigma_1,\sigma_2}$ be defined as in Lemma \ref{lem:duchi}. 
Then the gradient estimate error can be expanded as:
\begin{eqnarray*}
\E\left[\|\hbw \,-\, \nabla\psi_{\sigma_1}(\bell(\theta))\|^2\right] &\leq&
    2\E\left[\|\hbw \,-\, \nabla\psi_{\sigma_1,  \sigma_2}(\bell(\theta))\|^2\right] \,+\,
    2\|\nabla\psi_{\sigma_1,  \sigma_2}(\bell(\theta)) \,-\, \nabla\psi_{\sigma_1}(\bell(\theta))\|^2\\
    &\leq&
    4\E\left[\|\hbw_1 \,-\, \nabla\psi_{\sigma_1,  \sigma_2}(\bell(\theta))\|^2\right] \,+\,
    4\E\left[\|\hbw_2\|^2\right] \,+\,
    2\|\nabla\psi_{\sigma_1,  \sigma_2}(\bell(\theta)) \,-\, \nabla\psi_{\sigma_1}(\bell(\theta))\|^2\\
    &\leq&
    4\E\left[\|\hbw_1 \,-\, \nabla\psi_{\sigma_1,  \sigma_2}(\bell(\theta))\|^2\right]  \,+\,
    \frac{16\bar{\epsilon}^2}{ m\sigma_2^2}\sum_{j=1}^m \E\left[\|Z_2^j\|^2\right] \,+\,
    2\|\nabla\psi_{\sigma_1,  \sigma_2}(\bell(\theta)) \,-\, \nabla\psi_{\sigma_1}(\bell(\theta))\|^2 \\
    &\leq&
    \frac{CL^2K}{m}\left(\sqrt{\frac{\sigma_2}{\sigma_1}}K\,+\,\log K\,+\, 1\right) \,+\, 
    \frac{16\bar{\epsilon}^2K}{\sigma_2^2} \,+\,
    \frac{\sigma^2_2}{2}\frac{KL^2}{\sigma^2_1}(K+3)^3,
\end{eqnarray*}
where we used that (1) $\hbw_1$ is an unbiased estimate of $\nabla\psi_{\sigma_1, \sigma_2}(\bell(\theta))$ (see part 1 of Lemma \ref{lem:duchi}); (2) boundness assumption $|\epsilon(\theta)| \leq \bar{\epsilon}$; (3) $\|a_1+\cdots+a_m\|^2 \leq m(\|a_1\|^2+\cdots+\|a_m\|^2)$, and the last step follows from Parts 3--4 of Lemma \ref{lem:duchi}. 


Setting $\sigma_{2} = \sqrt{\frac{\sigma_1}{K^{3/2}L}}$ completes the proof.
\end{proof}

\subsection{Proof of Corollary \ref{cor:non-smooth}}
\begin{proof}
We begin by observing that convolution operation preserves monotonicity, convexity, and range of the function. Let $g_{\sigma_1}(\cdot)$ denotes Gaussian density function with variance $\sigma_1^2$, since $ \psi_{\sigma_1}(\bu)$ is a positively-weighted linear combination of shifted $\psi(\cdot)$, i.e., 
\[ \psi_{\sigma_1}(\bu) = \int _{\mathbb{R}^K} \psi(\bu-\bz) \cdot g_{\sigma_1}(\bz) \, d\bz=\int _{\mathbb{R}^K} \psi(\bz) \cdot g_{\sigma_1}(\bu-\bz) \, d\bz\, ,\]
Lipschitz property and convexity follows immediately from those on $\psi(\cdot)$. Moreover, since $g_{\sigma_1}$ is a probability distribution, we always have $\max |\psi_{\sigma_1}(\bu)| \leq \max |\psi(\bu)|$.
Taking derivatives, we have if $\psi(\cdot)$ is monotonic,
\[\frac{\partial \psi_{\sigma_1}(\bu)}{\partial u_i} = \nabla \psi_{\sigma_1}(\bu)^{\top}\mathbf{e}_i = \int _{\mathbb{R}^K} \nabla\psi(\bz)^{\top}\mathbf{e}_i \cdot g_{\sigma_1}(\bu-\bz) \, d\bz > 0\] 
therefore $\psi_{\sigma_1}(\cdot)$ is also monotonic.
Moreover, from Lemma~\ref{lem:duchi} we know $\psi_{\sigma_1}(\cdot)$ is smooth with parameter $\beta = \frac{\sqrt{K}L}{\sigma_1}$ and is $L$-Lipschitz, and that the mean-squared-error in gradient estimate $\hbw$ is bounded by $\kappa = \tilde{\cO}\Big(
\frac{L^{7/4}K^{13/8}}{m\sigma_{1}^{1/4}} \,+\, 
\frac{LK^{5/2}\bar{\epsilon}^2}{\sigma_1}
\Big)$ from Lemma~\ref{lem:fd-2-step}.
Applying 
Theorem \ref{thm:meta-result}
on the smoothed metric $\psi_{\sigma_1}(
\cdot)$ with $\eta =\frac{1}{\beta^2}= \frac{\sigma_1^2}{KL^2}$ then completes the proof.
\end{proof}

\section{Surrogate PGD as Optimizing a Linear Combination of Surrogates}
\label{app:prox}
In this section, we provide an interpretation of Algorithm \ref{algo:pgd} as optimizing an \textit{adaptively chosen} linear combination of the surrogates $\bell(\theta)$ with an additional proximal penalty like term. Recall that Step 6 of the surrogate projected gradient descent algorithm in Algorithm \ref{algo:pgd} solves the following optimization problem:
\begin{equation}
       \theta^{t+1} \in \argmin{\theta \in \thetaspace}\,\|\big(\bell(\theta) \,-\, \tilde{\bu}^{t+1}\big)_+\|^2.
\label{eq:pgd-step-6}
\end{equation}
\begin{lemma}
The optimization problem in \eqref{eq:pgd-step-6} is equivalent to: 
\[
\theta^{t+1} \in \argmin{\theta \in \thetaspace}\, \big\langle\bw^t,\bell(\theta)\big\rangle \,+\,
\mathbb{D}(\theta, \theta^t),
\]
where 
$
\mathbb{D}(\theta, \theta^t) \,=\,
\frac{1}{2\eta}\big\|\bell(\theta)-\bell(\theta^t)\big\|^2+\frac{1}{2\eta}\big\|(\bell(\theta) \,-\, \bell(\theta^t) \,+\, \eta\, \bw^t)_+\big\|^2-\frac{1}{2\eta}\big\|( \bell(\theta^t) \,-\, \eta\, \bw^t-\bell(\theta) \,)_+\big\|^2$.
\end{lemma}
Thus \eqref{eq:pgd-step-6} can be seen as minimizing a sum of linear combination of the surrogates and  (roughly speaking)  a term penalizing some form of distance between the current iterate $\theta^{t+1}$ and the previous iterate $\theta^{t}$.
\begin{proof}
Expanding the optimization problem in \eqref{eq:pgd-step-6}:
\begin{align*}
\theta^{t+1} \in \argmin{\theta \in \thetaspace}\,\|\big(\bell(\theta) \,-\, (\bell(\theta^t) \,-\, \eta\, \bw^t)\big)_+\|^2.
\end{align*}
Using the identity $(x)_+ = \frac{x+|x|}{2}$, we can write the objective in the above problem as
\begin{align*}
    \frac{1}{4}&\Big\|\bell(\theta) \,-\, \bell(\theta^t) \,+\, \eta\, \bw^t+ |\bell(\theta) \,-\, \bell(\theta^t) \,+\, \eta\, \bw^t| \Big\|^2\\
    &=\frac{1}{2}\Big\|\bell(\theta) \,-\, \bell(\theta^t) \,+\, \eta\, \bw^t\Big\|^2+\frac{1}{2}\Big\langle \bell(\theta) \,-\, \bell(\theta^t) \,+\, \eta\, \bw^t, |\bell(\theta) \,-\, \bell(\theta^t) \,+\, \eta\, \bw^t|\Big\rangle
    \end{align*}
which by ignoring constant terms and noticing that the second term is positive for the coordinates for which $\ell_k(\theta) > \ell_k(\theta^t) \,-\, \eta\, \bw_k^t$ and negative otherwise, we have that
\[\theta^{t+1} \in \argmin{\theta \in \thetaspace}\, \Big\langle\bw^t,\bell(\theta)\Big\rangle+\frac{1}{2\eta}\Big\|\bell(\theta)-\bell(\theta^t)\Big\|^2+\frac{1}{2\eta}\Big\|(\bell(\theta) \,-\, \bell(\theta^t) \,+\, \eta\, \bw^t)_+\Big\|^2-\frac{1}{2\eta}\Big\|( \bell(\theta^t) \,-\, \eta\, \bw^t-\bell(\theta) \,)_+\Big\|^2,\]
as desired.
\end{proof}

\section{Additional Experimental Details}
\label{app:expts}

\subsection{Choice of Hyper-parameters}
For the inner projection step in Algorithm \ref{algo:pgd}, we run Adagrad with a fixed step-size of 1.0 for 100 iterations. 
We used Adagrad as the optimization method for each of the baselines (including logistic regression, and the Relaxed F-measure approach and the Generalized Rates approach in Section \ref{sec:expt-fm}).
We tuned the hyper-parameters such as the step size $\eta$ for the proposed surrogate PGD algorithm and for the baseline Adagrad solvers, and the perturbation parameter $\sigma$ for gradient estimation in Algorithm \ref{algo:ls} using a held-out validation set. 

For the F-measure experiments in Section \ref{sec:expt-fm}, we chose the step sizes from the range $\{0.05, 0.1, 0.5, 1.0, 5.0\}$ and $\sigma$ from the range $\{0.05, 0.1, 0.5\}$.  For the ranking experiments in Section \ref{sec:expt-ranking}, we chose the step sizes from $\{0.001, 0.005, 0.01\}$ and found a fixed $\sigma$ of 1.5 to  work well across all runs. For the proxy label experiments in Section \ref{sec:expt-proxy}, we chose the step sizes from the range $\{0.01, 0.05, 0.1, 0.5, 1.0\}$ and $\sigma$ from the range $\{0.01, 0.05, 0.1, 0.5, 1.0\}$. 

For the larger KDD Cup 2008 dataset in the ranking experiments in Section \ref{sec:expt-ranking}, we used minibatches of size 100 to estimate gradients.

\subsection{Choice of Number of Perturbations}
In all our experiments, we chose to use 1000 perturbations to estimating gradients for $\psi$ in the proposed algorithm as this was a sufficiently large number that worked well across all experiments. But for many experiments, we could get comparable results with fewer perturbations. For example for the experiments in Sec 6.1, with as few as 10 perturbations, our approach achieved a test G-mean of 0.801, a comparable value to what we report for the proposed method  in Table \ref{tab:gmean} (0.803). Similarly, for the macro F-measure experiments in Table \ref{tab:fm-clean}, we got comparable results with just 10 perturbations. We report these comparisons in Table \ref{tab:fm-clean-perturbations}. For the larger KDD Cup 2008 dataset in the ranking experiments in Section \ref{sec:expt-ranking}, we estimated gradients with minibatches of size 100 and only perturbed the examples within a batch for the gradient computation.

\begin{table}[t]
\centering
\caption{Average test macro F-measure across groups with clean features. \textit{Higher} is better. We compare the results for the proposed method with 10 and 1000 perturbations to estimate gradients.}
\vspace{3pt}
\label{tab:fm-clean-perturbations}
\begin{tabular}{lccccc}
	\hline
& \#perturbations = 10 & \#perturbations = 1000 \\\hline
Business &0.796 &{0.796} \\
COMPAS	 &0.630 & 0.629 \\
Adult	 &0.661 & {0.665} \\
Default	 &0.532 & 0.533 \\
\hline
\end{tabular}
\end{table}
\begin{figure}[H]
    \centering
    \includegraphics[scale=0.4]{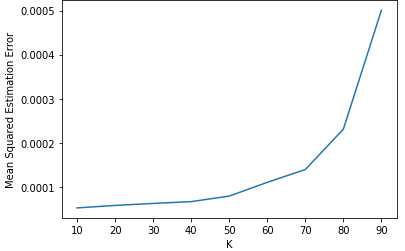}
    \caption{Mean squared estimation error for gradients estimated by the local linear interpolation approach in Algorithm \ref{algo:ls} for a synthetic $K$-dimensional gradient estimation problem, as $K$ varies.}
    \label{fig:dependence-k}
\end{figure}

\subsection{Dependence of the Gradient Estimation Error on $K$}
\label{sec:effect-of-k}
While the error bound for the linear interpolation based gradient estimation approach in Lemma \ref{thm:ls} has a strong dependence on the number of surrogates $K$, we find that in our simulations, the dependence on $K$ is less severe. This is evident from the plot shown in Figure \ref{fig:dependence-k}, where we consider the toy problem of estimating the gradient of the function $f(z) = \left(\prod_{k=1}^K z_k\right)^{1/K}$, where $z \in \R^{K}_+$, and we draw each coordinate $z_k$ from $0.1 + \text{Unif}(0, 0.9)$,  
We adopt the local linear interpolation based approach in Algorithm \ref{algo:ls} to estimate gradients for $f$ and evaluate the mean squared error for the gradient estimates w.r.t. the true gradient of $f$.
We use 100 perturbations, and report the average estimation errors over 100 random draws of $z$ and over 100 random trials for each draw of $z$. The figure plots the error as the input dimension $K$ varies.
~\\[7cm]
\end{document}

%% file: helpers.tex
\usepackage{amsmath,amssymb,amsthm}
\usepackage{color}
\usepackage{graphicx}
\usepackage{float}
\usepackage{algorithm}
\usepackage{algorithmic}
\usepackage{amsfonts}
\usepackage{bm}
\usepackage{subcaption}
\usepackage{enumerate}

\usepackage{natbib}
\newtheorem{theorem}{Theorem}
\newreptheorem{theorem}{Theorem}
\newtheorem{corollary}{Corollary}
\newtheorem{lemma}[theorem]{Lemma}
\newreptheorem{lemma}{Lemma}
\theoremstyle{remark}

\theoremstyle{definition}

\renewcommand{\>}{\rightarrow}

\newcommand{\E}{\mathbf{E}}

\renewcommand{\P}{\mathbf{P}}
\newcommand{\R}{\mathbb{R}}

\newcommand{\cX}{{\mathcal{X}}}
\newcommand{\cY}{{\mathcal{Y}}}

\newcommand{\cL}{{\mathcal{L}}}
\newcommand{\cU}{{\mathcal{U}}}

\newcommand{\cO}{\mathcal{O}}

\newcommand{\0}{{\mathbf{0}}}

\newcommand{\TPR}{\textrm{\textup{TPR}}}

\newcommand{\TNR}{\textrm{\textup{TNR}}}

\newcommand{\sign}{\textrm{\textup{sign}}}
\newcommand{\bz}{\mathbf{z}}
\newcommand{\bL}{\mathbf{L}}
\newcommand{\bH}{\mathbf{H}}

\newcommand{\argmin}[1]{\underset{#1}{\operatorname{argmin}}}
\newcommand{\amin}[1]{\operatorname{argmin}_{#1}}

\newcommand{\bell}{\boldsymbol{\ell}}
\newcommand{\bw}{\hat{\mathbf{g}}}
\newcommand{\hbw}{\hat{\mathbf{g}}}

\newcommand{\bu}{\mathbf{u}}
\newcommand{\by}{\mathbf{y}}
\newcommand{\bx}{\mathbf{x}}
\newcommand{\boldf}{\mathbf{f}}
\newcommand{\thetaspace}{\R^d}

\newcommand{\boldI}{\mathbf{I}}
\newcommand{\boldO}{\mathbf{0}}
\newcommand{\bM}{\mathbf{M}}
\newcommand{\truncg}{\textrm{\textup{trunc}}(\bw)}